\documentclass[10pt,twocolumn,letterpaper]{article}

\makeatletter
\def\@fnsymbol#1{\ensuremath{\ifcase#1\or \dagger\or \ddagger\or
   \mathsection\or \mathparagraph\or \|\or **\or \dagger\dagger
   \or \ddagger\ddagger \else\@ctrerr\fi}}
\makeatother
\usepackage[pagenumbers]{cvpr} 

\usepackage[dvipsnames]{xcolor}

\usepackage{graphicx}
\usepackage{wrapfig}
\usepackage{amsmath}
\usepackage{amssymb}
\usepackage{amsthm}
\usepackage{booktabs}
\usepackage{subcaption}
\usepackage{mathtools}
\usepackage{pifont}
\usepackage{cancel}
\usepackage{algorithm}
\usepackage{algorithmic}
\usepackage{multirow}
\usepackage{colortbl}
\usepackage{bbding}
\usepackage{utfsym}

\newtheorem{proposition}{Proposition}
\newtheorem{assumption}{Assumption}
\newtheorem{lemma}{Lemma}

\def\bff{\mathbf{f}}

\def\bfh{\mathbf{h}}

\def\bfs{\mathbf{s}}

\def\bfw{\mathbf{w}}
\def\bfx{\mathbf{x}}
\def\bfy{\mathbf{y}}
\def\bfz{\mathbf{z}}

\def\bfD{\mathbf{D}}
\def\bfI{\mathbf{I}}

\def\bftheta{\boldsymbol{\theta}}
\def\bfeps{\boldsymbol{\epsilon}}

\def\bfomega{\boldsymbol{\omega}}

\def\nablax{\nabla_{\bfx}}

\def\rmd{\mathrm{d}}
\def\rmdx{\mathrm{d}\bfx}

\def\bbE{\mathbb{E}}
\def\bbR{\mathbb{R}}

\def\bbP{\mathbb{P}}

\def\calN{\mathcal{N}}

\def\eqref#1{Eq.~(\ref{#1})}

\definecolor{cvprblue}{rgb}{0.21,0.49,0.74}
\usepackage[pagebackref,breaklinks,colorlinks,citecolor=cvprblue]{hyperref}

\def\paperID{7035} 
\def\confName{CVPR}
\def\confYear{2024}

\title{Fast ODE-based Sampling for Diffusion Models in Around 5 Steps}

\author{
    Zhenyu Zhou\qquad Defang Chen\thanks{Corresponding author}\qquad Can Wang\qquad Chun Chen\\
    The State Key Laboratory of Blockchain and Data Security, \\ Zhejiang University, China\\
    {\tt\small \{zhyzhou, defchern, wcan, chenc\}@zju.edu.cn}
}

\begin{document}
\maketitle
\begin{abstract}
Sampling from diffusion models can be treated as solving the corresponding ordinary differential equations (ODEs), with the aim of obtaining an accurate solution with as few number of function evaluations (NFE) as possible. 
Recently, various fast samplers utilizing higher-order ODE solvers have emerged and achieved better performance than the initial first-order one.
However, these numerical methods inherently result in certain approximation errors, which significantly degrades sample quality with extremely small NFE (\textit{e.g.}, around 5).
In contrast, based on the geometric observation that each sampling trajectory almost lies in a two-dimensional subspace embedded in the ambient space, we propose \textbf{A}pproximate \textbf{ME}an-\textbf{D}irection Solver (AMED-Solver) that eliminates truncation errors by directly learning the mean direction for fast diffusion sampling. 
Besides, our method can be easily used as a plugin to further improve existing ODE-based samplers. Extensive experiments on image synthesis with the resolution ranging from 32 to 512 demonstrate the effectiveness of our method. With only 5 NFE, we achieve 6.61 FID on CIFAR-10, 10.74 FID on ImageNet 64$\times$64, and 13.20 FID on LSUN Bedroom. Our code is available at \url{https://github.com/zju-pi/diff-sampler}.

\end{abstract}    
\section{Introduction}
\label{sec:intro}
Diffusion models have been attracting growing attentions in recent years due to their impressive generative capability~\cite{dhariwal2021diffusion,rombach2022ldm,saharia2022photorealistic,ruiz2023dreambooth}. Given a noise input, they 
are able to generate a realistic output by performing iterative denoising steps with the score function~\cite{sohl2015deep,ho2020ddpm,song2021sde}. This process can be interpreted as applying a certain numerical discretization on a stochastic differential equation (SDE), or more commonly, its corresponding probability flow ordinary differential equation (PF-ODE)~\cite{song2021sde}. Comparing to other generative models such as GANs~\cite{goodfellow2014generative} and VAEs~\cite{kingma2013auto}, diffusion models have the advantages in high sample quality and stable training, but suffer from slow sampling speed, which poses a great challenge to their applications.
\begin{figure}
  \centering
    \includegraphics[width=\linewidth]{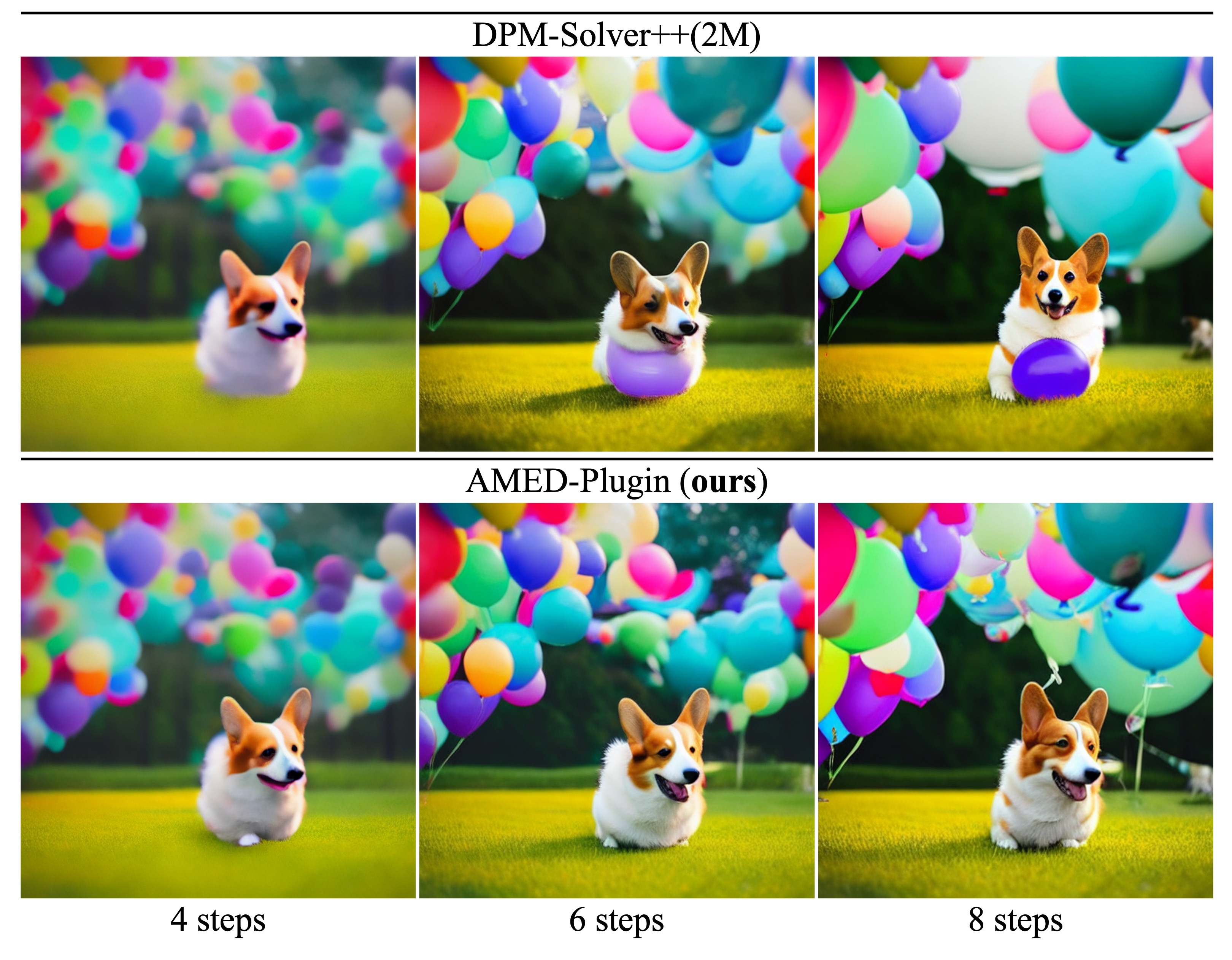}
    \caption{Synthesized images by Stable-Diffusion~\cite{rombach2022ldm} with a default classifier-free guidance scale 7.5 and a text prompt ``\textit{A Corgi on the grass surrounded by a cluster of colorful balloons}''. Our method improves DPM-Solver++(2M)~\cite{lu2022dpmpp} in sample quality.}
    \label{fig:teaser}
\end{figure}

Existing methods for accelerating diffusion sampling fall into two main streams. One is designing faster numerical solvers to increase step size while maintaining small truncation errors~\cite{song2021ddim,liu2022pseudo,lu2022dpm,zhang2023deis,karras2022edm,dockhorn2022genie}. They can be further categorized as \textit{single-step} solvers and \textit{multi-step} solvers~\cite{atkinson2011numerical}. The former computes the next step solution only using information from the current time step, while the latter uses multiple past time steps. These methods have successfully reduced the number of function evaluations (NFE) from 1000 to less than 20, almost without affecting the sample quality. 
\begin{figure*}[t]
    \centering
    \begin{subfigure}[b]{0.24\textwidth}
        \includegraphics[width=\textwidth]{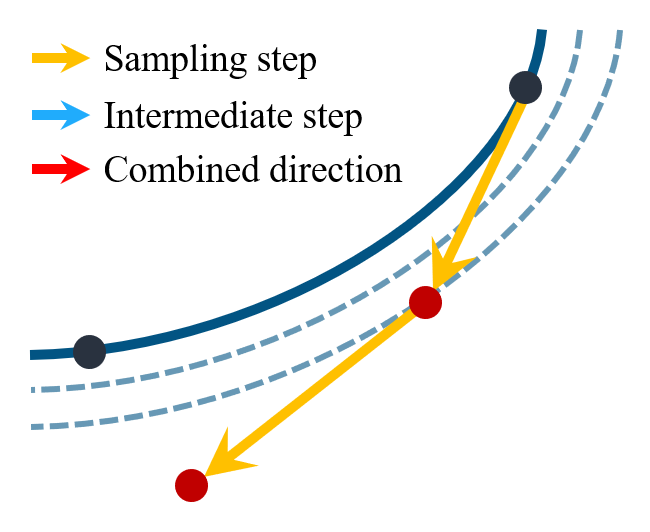}
        \caption{DDIM solver.}
    \end{subfigure}
    \begin{subfigure}[b]{0.24\textwidth}
        \includegraphics[width=\textwidth]{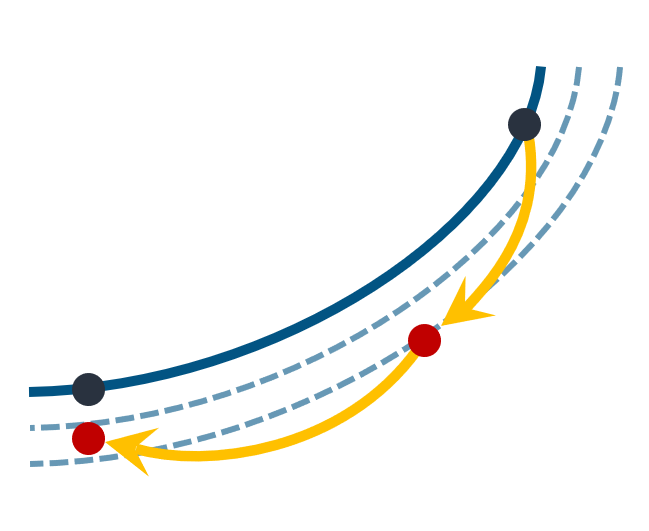}
        \caption{Multi-step solvers.}
    \end{subfigure}
    \begin{subfigure}[b]{0.24\textwidth}
        \includegraphics[width=\textwidth]{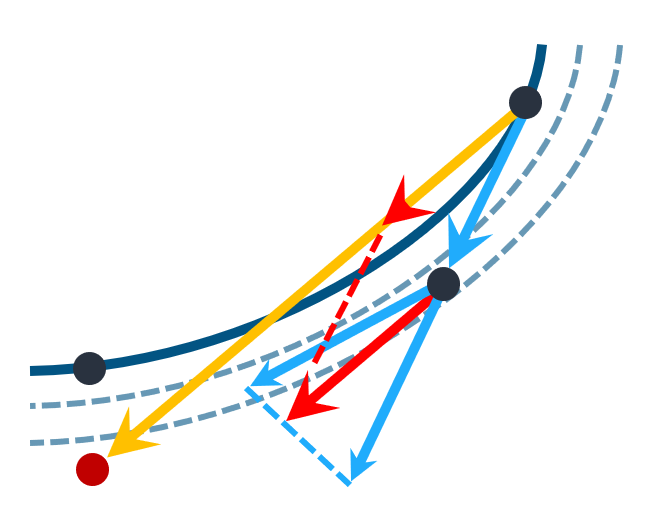}
        \caption{Generalized DPM-Solver-2.}
    \end{subfigure}
    \begin{subfigure}[b]{0.24\textwidth}
        \includegraphics[width=\textwidth]{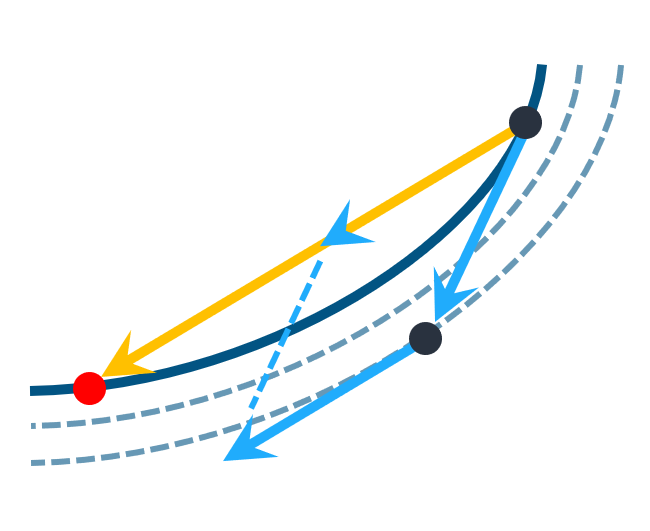}
        \caption{AMED-Solver (ours).}
    \end{subfigure}
    \caption{Comparison of various ODE solvers. Red dots depict the actual sampling step of different solvers. 
    (a) DDIM solver~\cite{song2021ddim} applies Euler discretization on PF-ODEs. In every sampling step, it follows the gradient direction to give the solution for next time step. 
    (b) Multi-step solvers~\cite{liu2022pseudo,zhang2023deis,lu2022dpmpp,zhao2023unipc} require current gradient and several records of history gradients and then follow the combination of these gradients to give the solution. 
    (c) In generalized DPM-Solver-2~\cite{lu2022dpm}, there is a hyper-parameter $r$ controlling the location of intermediate time step. $r=0.5$ recovers the default DPM-Solver-2 and $r=1$ recovers Heun's second method~\cite{karras2022edm}. 
    The gradient for sampling step is given by the combination of gradients at intermediate and current time steps (see \cref{tab:comparison}). 
    (d) Our proposed AMED-Solver seeks to find the intermediate time step and the scaling factor that gives nearly optimal gradient directing to the ground truth solution. This gradient used for sampling step is adaptively learned instead of the heuristic assigned as DPM-Solver-2.}
    \label{fig:method_comparison}
\end{figure*}
Another kind of methods aim to build a one-to-one mapping between the data distribution and the pre-specified noise distribution~\cite{luhman2021knowledge,salimans2022progressive,liu2022flow,song2023consistency,berthelot2023tract}, based on the idea of knowledge distillation. With a well-trained student model in hand, high-quality generation can be achieved with only one NFE. However, training such a student model either requires pre-generation of millions of images~\cite{luhman2021knowledge,liu2022flow}, or huge training cost with carefully modified training procedure~\cite{salimans2022progressive,song2023consistency, berthelot2023tract}. Besides, distillation-based models cannot guarantee the increase of sample quality given more NFE and they have difficulty in likelihood evaluation.

In this paper, we further boost ODE-based sampling for diffusion models in around 5 steps. Based on the geometric property that each sampling trajectory approximately lies in a two-dimensional subspace embedded in the high-dimension space, we propose \textit{\textbf{A}pproximate \textbf{ME}an-\textbf{D}irection Solver} (AMED-Solver), a single-step ODE solver that learns to predict the mean direction
in every sampling step. 
A comparison of various ODE solvers is illustrated in \cref{fig:method_comparison}. We also extend our method to any ODE solvers as a plugin.
When applying AMED-Plugin on the improved PNDM solver~\cite{zhang2023deis}, we achieve FID of 6.61 on CIFAR-10, 10.74 on ImageNet 64$\times$64, and 13.20 on LSUN Bedroom. 
Our main contributions are as follows:
\begin{itemize}[leftmargin=2em]
    \item We introduce AMED-Solver, a new single-step ODE solver for diffusion models that eliminates truncation errors by design.
    \item We propose AMED-Plugin that can be applied to any ODE solvers with a small training overhead and a negligible sampling overhead.
    \item Extensive experiments on various datasets validate the effectiveness of our method in fast image generation. 
\end{itemize}

\section{Background}
\label{sec:background}

\subsection{Diffusion Models}
\label{subsec:diffusion models}
The forward diffusion process can be formalized as a SDE:
\begin{equation}
    \label{eq:forward_sde}
    \rmdx = \bff(\bfx, t)\rmd t + g(t) \rmd \bfw_t,
\end{equation}
where $\bff(\cdot, t): \bbR^d \rightarrow \bbR^d, g(\cdot): \bbR \rightarrow \bbR$ are drift and diffusion coefficients, respectively, and $\bfw_t \in \bbR^d$ is the standard Wiener process~\cite{oksendal2013stochastic}. 
This forward process forms a continuous stochastic process $\{\bfx_t\}_{t=0}^T$ and the associated probability density $\{p_t(\bfx)\}_{t=0}^T$, to make the sample $\bfx_0$ from the implicit data distribution $p_d=p_0$ approximately distribute as the pre-specified noise distribution, \ie, $p_T\approx p_n$. 
Given an encoding $\bfx_T \sim p_n$, generation is then performed with the reversal of \cref{eq:forward_sde}~\cite{feller1949theory,anderson1982reverse}. Remarkably, there exists a probability flow ODE (PF-ODE) 
\begin{equation}
    \label{eq:pf_ode}
    \rmdx = \left[\bff(\bfx, t) - \frac{1}{2} g(t)^2 \nablax \log p_t(\bfx)\right]\rmd t, 
\end{equation}
sharing the same marginals with the reverse SDE~\cite{maoutsa2020interacting,song2021sde}, and $\nablax \log p_t(\bfx)$ is known as the \textit{score function}~\cite{hyvarinen2005estimation,lyu2009interpretation}. Generally, this PF-ODE is preferred in practice for its conceptual simplicity, sampling efficiency and unique encoding~\cite{song2021sde}. Throughout this paper, we follow the configuration of EDM~\cite{karras2022edm} by setting $\bff(\bfx, t)=\mathbf{0}$, $g(t)=\sqrt{2t}$ and $\sigma(t)=t$. In this case, the reciprocal of $t^2$ equals to the \textit{signal-to-noise ratio}~\cite{kingma2021vdm} and the perturbation kernel is
\begin{equation}
    \label{eq:perturbation_kernel}
    p_t(\bfx|\bfx_0)=\calN(\bfx; \bfx_0, t^2\bfI).
\end{equation}
To simulate the PF-ODE, we usually train a U-Net~\cite{ronneberger2015u,ho2020ddpm} predicting $\bfs_{\theta}(\bfx, t)$ to approximate the intractable $\nablax \log p_t(\bfx)$.
There are mainly two parameterizations in the literature. One uses a noise prediction model $\bfeps_{\theta}(\bfx, t)$ predicting the Gaussian noise added to $\bfx$ at time $t$~\cite{ho2020ddpm,song2021ddim}, and another uses a data prediction model $\bfD_{\theta}(\bfx, t)$ predicting the denoising output of $\bfx$ from time $t$ to $0$~\cite{karras2022edm,lu2022dpmpp,chen2023geometric}. 
They have the following relationship in our setting:
\begin{equation}
    \label{eq:relationship}
    \bfs_{\theta}(\bfx, t) = -\frac{\bfeps_{\theta}(\bfx, t)}{t} = \frac{\bfD_{\theta}(\bfx, t) - \bfx}{t^2}.
\end{equation}
The training of diffusion models in the noise prediction notation is performed by minimizing a weighted combination of the least squares estimations:
\begin{equation}
    \label{eq:trainning_loss}
    \bbE_{\bfx \sim p_d, \bfeps \sim \calN(\mathbf{0}, \bfI)} \lVert \bfeps_{\bftheta}(\bfx + t\bfeps; t) - \bfeps \rVert^2_2.
\end{equation}
We then plug the learned score function \cref{eq:relationship} into \cref{eq:pf_ode} to obtain a simple formulation for the PF-ODE
\begin{equation}
    \label{eq:plug-in ODE}
    \rmdx = \bfeps_{\theta}(\bfx, t)\rmd t.
\end{equation}
The \textit{sampling trajectory} $\{\bfx_{t_n}\}_{n=1}^N$ is obtained by first draw $\bfx_{t_N} \sim p_n = \calN(\mathbf{0}, T^2\bfI)$ and then solve \cref{eq:plug-in ODE} with $N-1$ steps following time schedule $\Gamma = \{t_1=\epsilon, \cdots, t_N=T\}$.

\subsection{Categorization of Previous Fast ODE Solvers}
\label{subsec:solvers}
To accelerate diffusion sampling, various fast ODE solvers have been proposed, which can be classified into \textit{single-step solvers} or \textit{multi-step solvers}~\cite{atkinson2011numerical}. 
Single-step solvers including DDIM~\cite{song2021ddim}, EDM~\cite{karras2022edm} and DPM-Solver~\cite{lu2022dpm} which only use the information from the current time step to compute the solution for the next time step, while multi-step solvers including PNDM~\cite{liu2022pseudo} and DEIS~\cite{zhang2023deis} utilize multiple past time steps to compute the next time step (see \cref{fig:method_comparison} for an intuitive comparison). We emphasize that one should differ single-step ODE solvers from single-step (NFE=1) distillation-based methods~\cite{luhman2021knowledge,salimans2022progressive,song2023consistency}.

The advantages of single-step methods lie in the easy implementation and the ability to self-start since no history record is required. However, as illustrated in \cref{fig:degradation}, they suffer from fast degradation of sample quality especially when the NFE budget is limited. The reason may be that the actual sampling steps of multi-step solvers are twice as much as those of single-step solvers with the same NFE, enabling them to adjust directions more frequently and flexibly. We will show that our proposed AMED-Solver can largely fix this issue with learned mean directions.

\begin{figure}
    \centering
    \includegraphics[width=\linewidth]{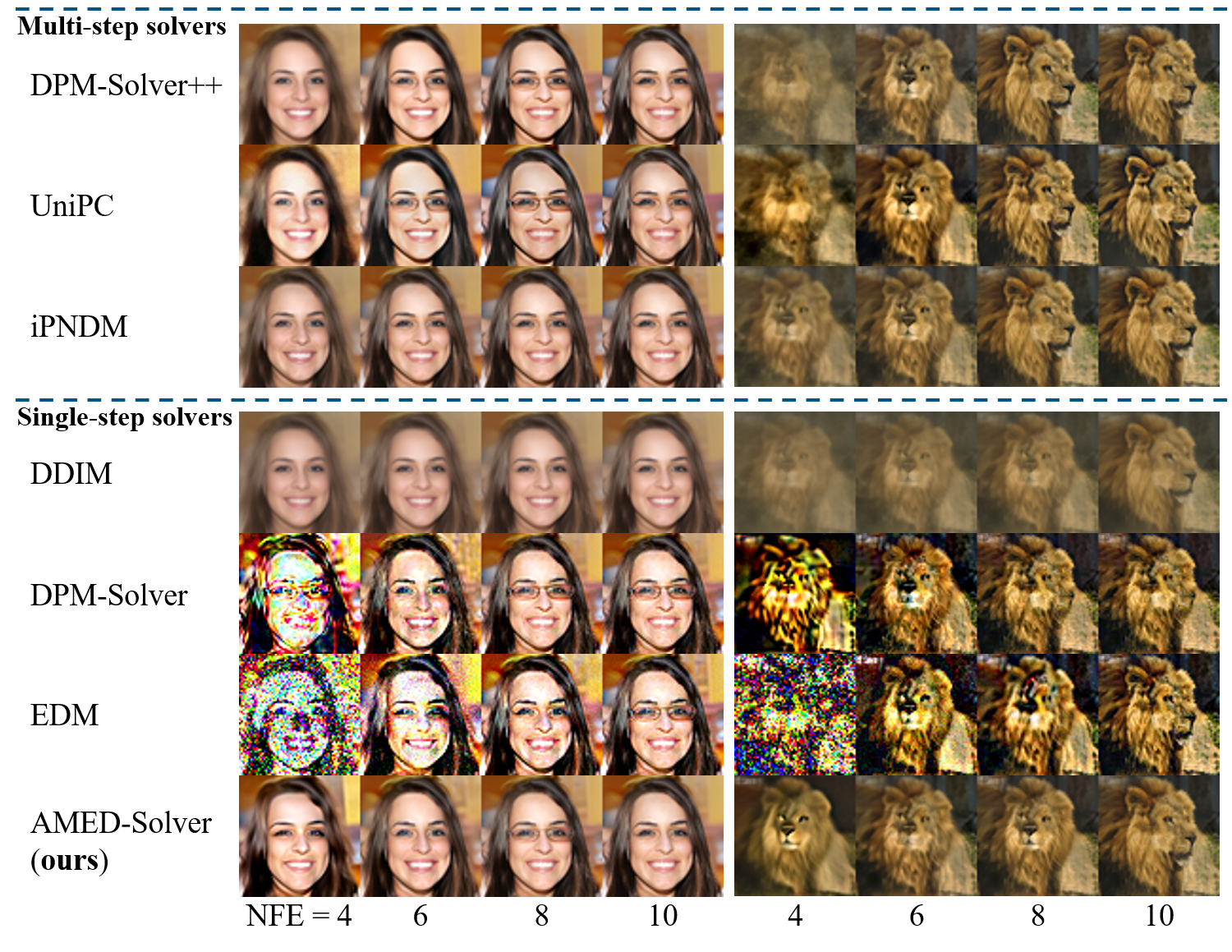}
    \caption{The sample quality degradation of multi-step and single-step ODE solvers. The quality of images generated by single-step solvers, especially higher-order ones including DPM-Solver-2~\cite{lu2022dpm} and EDM~\cite{karras2022edm}, rapidly decreases as NFE decreases, while our proposed AMED-Solver largely mitigates such degradation. Examples are from FFHQ 64$\times$64~\cite{karras2019style} and ImageNet 64$\times$64~\cite{russakovsky2015ImageNet}.}
    \label{fig:degradation}
\end{figure}

\begin{table*}
  \centering
  \begin{tabular}{@{}llll@{}}
    \toprule
    Method & Gradient term & Source of $s_n$ & Source of $c_n$ \\
    \midrule
    DDIM~\cite{song2021ddim} & $c_n\bfeps_{\theta}(\bfx_{t_{n+1}}, t_{n+1})$ & - & $1$ \\
    EDM~\cite{karras2022edm} & $c_n\left(\frac{1}{2} \bfeps_{\theta}(\bfx_{s_n}, s_n) + \frac{1}{2} \bfeps_{\theta}(\bfx_{t_{n+1}}, t_{n+1}) \right)$ & $t_n$ & $1$ \\
    Generalized DPM-Solver-2~\cite{lu2022dpm} & $c_n\left(\frac{1}{2r} \bfeps_{\theta}(\bfx_{s_n}, s_n) + (1-\frac{1}{2r}) \bfeps_{\theta}(\bfx_{t_{n+1}}, t_{n+1})\right)$ & $t_n^r t_{n+1}^{1-r}$, $r\in(0, 1]$ & $1$ \\
    AMED-Solver (ours) & $c_n\bfeps_{\theta}(\bfx_{s_n}, s_n)$ & Learned & Learned \\
    \bottomrule
  \end{tabular}
  \caption{Comparison of various single-step ODE solvers.}
  \label{tab:comparison}
\end{table*}

\section{Our Proposed AMED-Solver}
\label{sec:method}
In this section, we propose AMED-Solver, a single-step ODE solver for diffusion models that releases the potential of single-step solvers in extremely small NFE, enabling them to match or even surpass the performance of multi-step solvers. Furthermore, our proposed method can be generalized as a plugin on any ODE solver, yielding promising improvement across various datasets. Our key observation is that the sampling trajectory generated by \cref{eq:plug-in ODE} nearly lies in a two-dimensional subspace embedded in high-dimensional space, which motivates us to minimize the discretization error with the mean value theorem.

\subsection{The Sampling Trajectory Almost Lies in a Two-Dimensional Subspace}
\label{subsec:geometric}
The sampling trajectory generated by solving \cref{eq:plug-in ODE} exhibits an extremely simple geometric structure and has an implicit connection to annealed mean shift, as revealed in the previous work~\cite{chen2023geometric}. 
Each sample starting from the noise distribution approaches the data manifold along a smooth, quasi-linear sampling trajectory in a monotone likelihood increasing way. 
Besides, all trajectories from different initial points share the similar geometric shape, whether in the conditional or unconditional generation case~\cite{chen2023geometric}.

In this paper, we further point out that the sampling trajectory generated by ODE solvers almost lies in a two-dimensional plane embedded in a high-dimensional space. We experimentally validate this claim by performing Principal Component Analysis (PCA) for 1k sampling trajectories on different datasets including CIFAR10 32$\times$32~\cite{krizhevsky2009learning}, FFHQ 64$\times$64~\cite{karras2019style}, ImageNet 64$\times$64~\cite{russakovsky2015ImageNet} and LSUN Bedroom 256$\times$256~\cite{yu2015lsun}. As illustrated in \cref{fig:pca}, the relative projection error using two principal components is no more than 8$\%$ and always stays in a small level. Besides, the sample variance can be fully explained using only two principal components. Given the vast image space with dimensions of 3072 (3$\times$32$\times$32), 12288 (3$\times$64$\times$64), or 196608 (3$\times$256$\times$256), the sampling trajectories show intriguing property that their dynamics can almost be described using only two principal components.

\begin{figure}[t]
    \centering
    \begin{subfigure}[b]{0.23\textwidth}
        \includegraphics[width=\textwidth]{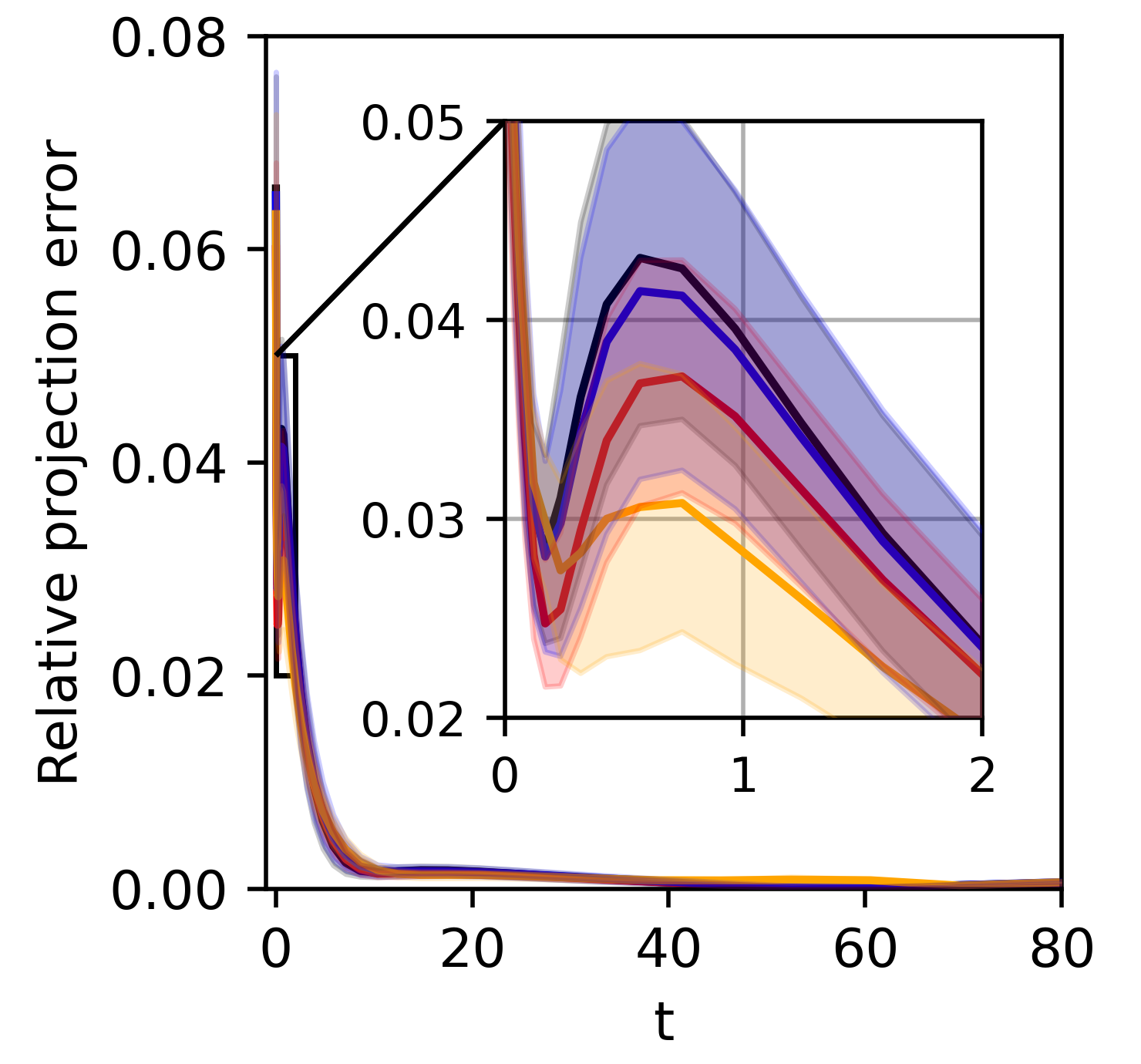}
        \caption{Relative deviation.}
    \end{subfigure}
    \begin{subfigure}[b]{0.23\textwidth}
        \includegraphics[width=\textwidth]{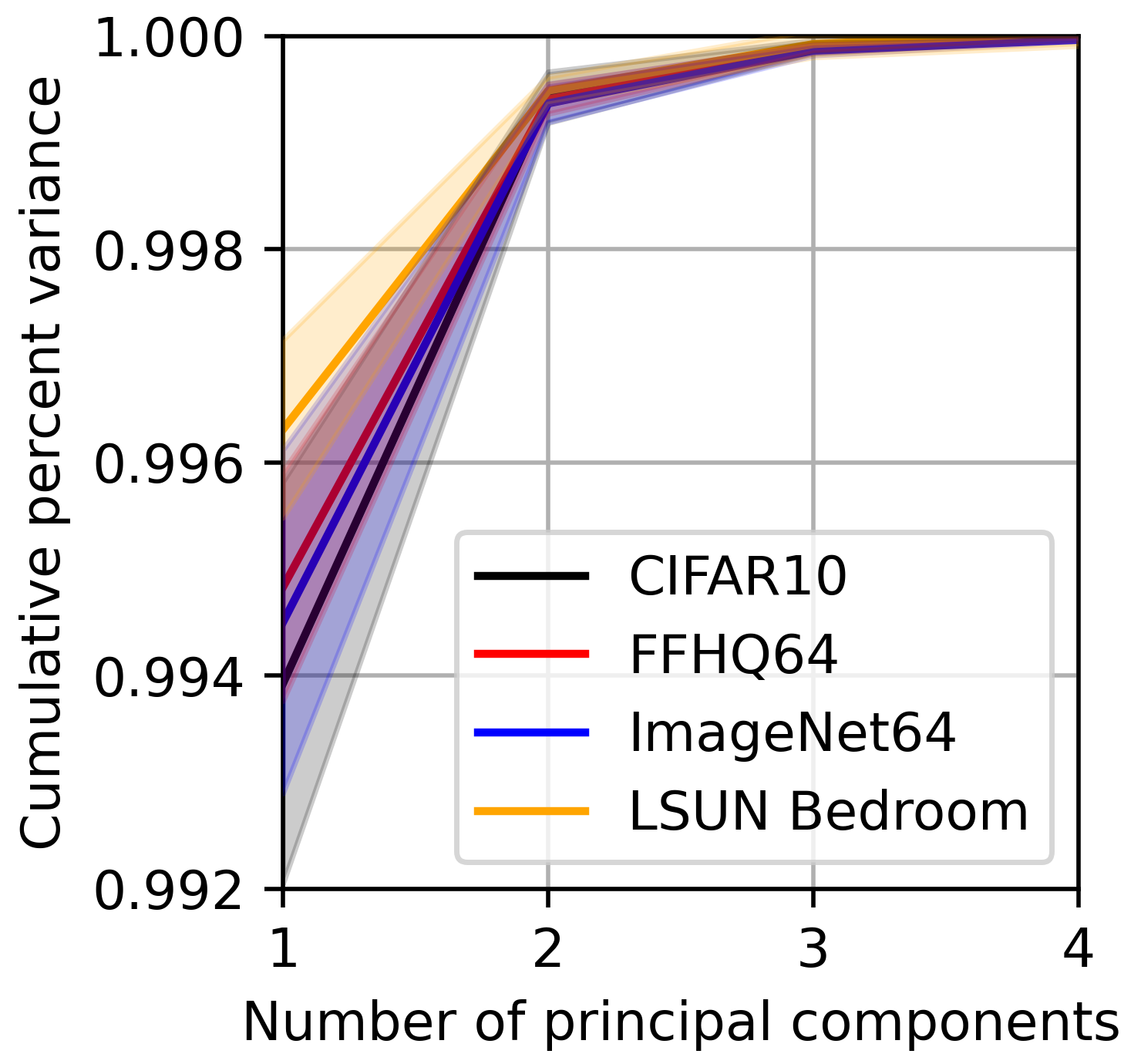}
        \caption{Variance explained by top PCs.}
    \end{subfigure}
    \caption{
    We perform PCA to each sampling trajectory $\{\bfx_{t}\}_{t=\epsilon}^{T}$.
    (a) These trajectories are projected into a 2D subspace spanned by the top 2 principal components to get $\{\tilde{\bfx}_{t}\}_{t=\epsilon}^T$ and the relative projection error is calculated as $\left \| \bfx_{t} - \tilde{\bfx}_{t} \right \|_2 / \left \| \bfx_{t} \right \|_2$. 
    (b) We progressively increase the number of principal components and calculate the cumulative percent variance as $\mathrm{Var}(\{\tilde{\bfx}_{t}\}_{t=\epsilon}^T) / \mathrm{Var}(\{\bfx_{t}\}_{t=\epsilon}^T)$. The results are obtained by averaging 1k sampling trajectories using EDM solver~\cite{karras2022edm} with 80 NFE.
    }
    \label{fig:pca}
\end{figure}

\subsection{Approximate Mean-Direction Solver}
\label{subsec:AMED-Solver}
With the geometric intuition, we next explain our methods in more detail. The exact solution of \cref{eq:plug-in ODE} is:
\begin{equation}
    \label{eq:solution}
    \bfx_{t_n} = \bfx_{t_{n+1}} + \int_{t_{n+1}}^{t_n} \bfeps_{\theta}(\bfx_t, t)\rmd t.
\end{equation}
Various numerical approximations to the integral above correspond to different types of fast ODE-based solvers. For instance, direct explicit rectangle method gives DDIM~\cite{song2021ddim}, linear multi-step method yields PNDM~\cite{liu2022pseudo}, Taylor expansion yields DPM-Solver~\cite{lu2022dpm} and polynomial extrapolation recovers DEIS~\cite{zhang2023deis}. Different from these works, we derive our method more directly by expecting that the \textit{mean value theorem} holds for the integral involved so that we can find an intermediate time step $s_n \in \left(t_n, t_{n+1}\right)$ and a scaling factor $c_n \in \bbR$ that satisfy
\begin{equation}
    \label{eq:approx1}
    \bfeps_{\theta}(\bfx_{s_n}, s_n) = \frac{1}{c_n(t_n - t_{n+1})} \int_{t_{n+1}}^{t_n} \bfeps_{\theta}(\bfx_t, t)\rmd t
\end{equation}
Although the well-known mean value theorem for real-valued functions does not hold in vector-valued case~\cite{cheney2001analysis}, the remarkable geometric property that the sampling trajectory $\{\bfx_{t}\}_{t=\epsilon}^T$ almost lies in a two-dimensional subspace guarantees our use. 
By properly choosing $s_n$ and $c_n$, we are able to achieve an approximation of \cref{eq:solution} by
\begin{equation}
    \label{eq:approx2}
    \bfx_{t_n} \approx \bfx_{t_{n+1}} + c_n (t_n - t_{n+1}) \bfeps_{\theta}(\bfx_{s_n}, s_n).
\end{equation}
This formulation gives a single-step ODE solver. The DPM-Solver-2~\cite{lu2022dpm} can be recovered by setting $s_n = \sqrt{t_n t_{n+1}}$ and $c_n=1$. In \cref{tab:comparison}, we compare various single-step solvers by generalizing $c_n\bfeps_{\theta}(\bfx_{s_n}, s_n)$ as the \textit{gradient term}.

For the choose of $\{s_n\}_{n=1}^{N-1}$ and $\{c_n\}_{n=1}^{N-1}$, we train a shallow neural network $g_{\phi}$ (named as \textit{AMED predictor}) based on distillation with small training and negligible sampling costs. Briefly, given samples $\bfy_{t_n}, \bfy_{t_{n+1}}$ on the teacher sampling trajectory and $\bfx_{t_{n+1}}$ on the student sampling trajectory, $g_{\phi}$ gives $s_n$ and $c_n$ that minimizes $d\left(\bfx_{t_n}, \bfy_{t_n}\right)$ where $\bfx_{t_n}$ is given by \cref{eq:approx2} and $d(\cdot,\cdot)$ is a distance metric. Since we seek to find a mean-direction that best approximates the integral in \cref{eq:approx1}, we name our proposed single-step ODE solver \cref{eq:approx2} as \textit{\textbf{A}pproximate \textbf{ME}an-\textbf{D}irection Solver}, dubbed as AMED-Solver. Before specifying the training and sampling details, we proceed to generalize our idea as a plugin to the existing fast ODE solvers.

\subsection{AMED as A Plugin}
\label{subsec:AMED-Plugin}
The idea of AMED can be used to further improve existing fast ODE solvers for diffusion models.
Throughout our analysis, we take the polynomial schedule~\cite{karras2022edm} as example:
\begin{equation}
    \label{eq:poly_schedule}
    t_n=(t_1^{1/\rho}+\frac{n-1}{N-1}(t_N^{1/\rho}-t_1^{1/\rho}))^{\rho}.
\end{equation}
Note that another usually used uniform logSNR schedule is actually the limit of \cref{eq:poly_schedule} as $\rho$ approaches $+\infty$. 

Given a time schedule $\Gamma = \{t_1=\epsilon, \cdots, t_N=T\}$, the AMED-Solver is obtained by performing extra model evaluations at $s_n \in (t_n, t_{n+1}), n=1,\cdots,N-1$. Under the same manner, we are able to improve any ODE solvers by predicting $\{s_n\}_{n=1}^{N-1}$ and $\{c_n\}_{n=1}^{N-1}$ that best aligns the student and teacher sampling trajectories.

We validate this by an experiment where we fix $c_n=1$ and first generate a ground truth trajectory $\{\bfx_{t_n}^G\}_{n=1}^N$ using Heun's second method with 80 NFE and extract samples at $\Gamma$. For every ODE solver, we generate a baseline trajectory $\{\bfx_{t_n}^B\}_{n=1}^N$ by performing evaluations at $s_n=\sqrt{t_nt_{n+1}}$ as DPM-Solver-2~\cite{lu2022dpm} does.
We then apply a grid search on $r_n$, giving $s_n=t_n^{r_n} t_{n+1}^{1-{r_n}}$ and a searched trajectory $\{\bfx_{t_n}^S\}_{n=1}^N$. We define the relative alignment to be $\left \| \bfx_{t_n}^B - \bfx_{t_n}^G \right \|_2 - \left \| \bfx_{t_n}^S - \bfx_{t_n}^G \right \|_2$. 
Positive relative alignment value indicates that the searched trajectory $\{\bfx_{t_n}^S\}_{n=1}^N$ is closer to the ground truth trajectory $\{\bfx_{t_n}^G\}_{n=1}^N$ than the baseline trajectory $\{\bfx_{t_n}^B\}_{n=1}^N$.
As shown in \cref{fig:align}, the relative alignment value keeps positive in most cases, meaning that appropriate choose of intermediate time steps can further improve fast ODE solvers. Therefore, as described in \cref{subsec:AMED-Solver}, we also train an AMED predictor to predict the intermediate time steps as well as the direction scaling factor. As this process still has the meaning of searching for the direction pointing to the ground truth, we name this method as \textit{AMED-Plugin} and apply it on various fast ODE solvers.

Through \cref{fig:align}, we obtain a direct comparison between DPM-Solver-2 and our proposed AMED-Solver since they share the same baseline trajectory when $r$ is fixed to $0.5$. Our AMED-Solver aligns better with the ground truth and the location of searched intermediate time steps is much more stable than that of DPM-Solver-2. We speculate that the gradient direction of DPM-Solver-2 restricted by the fixed $r$ (see \cref{tab:comparison}) is suboptimal. Instead, 
the learned coefficients provide AMED-Solver more flexibility to determine a better gradient direction.

\begin{figure}[t]
    \centering
    \begin{subfigure}[b]{0.23\textwidth}
        \includegraphics[width=\textwidth]{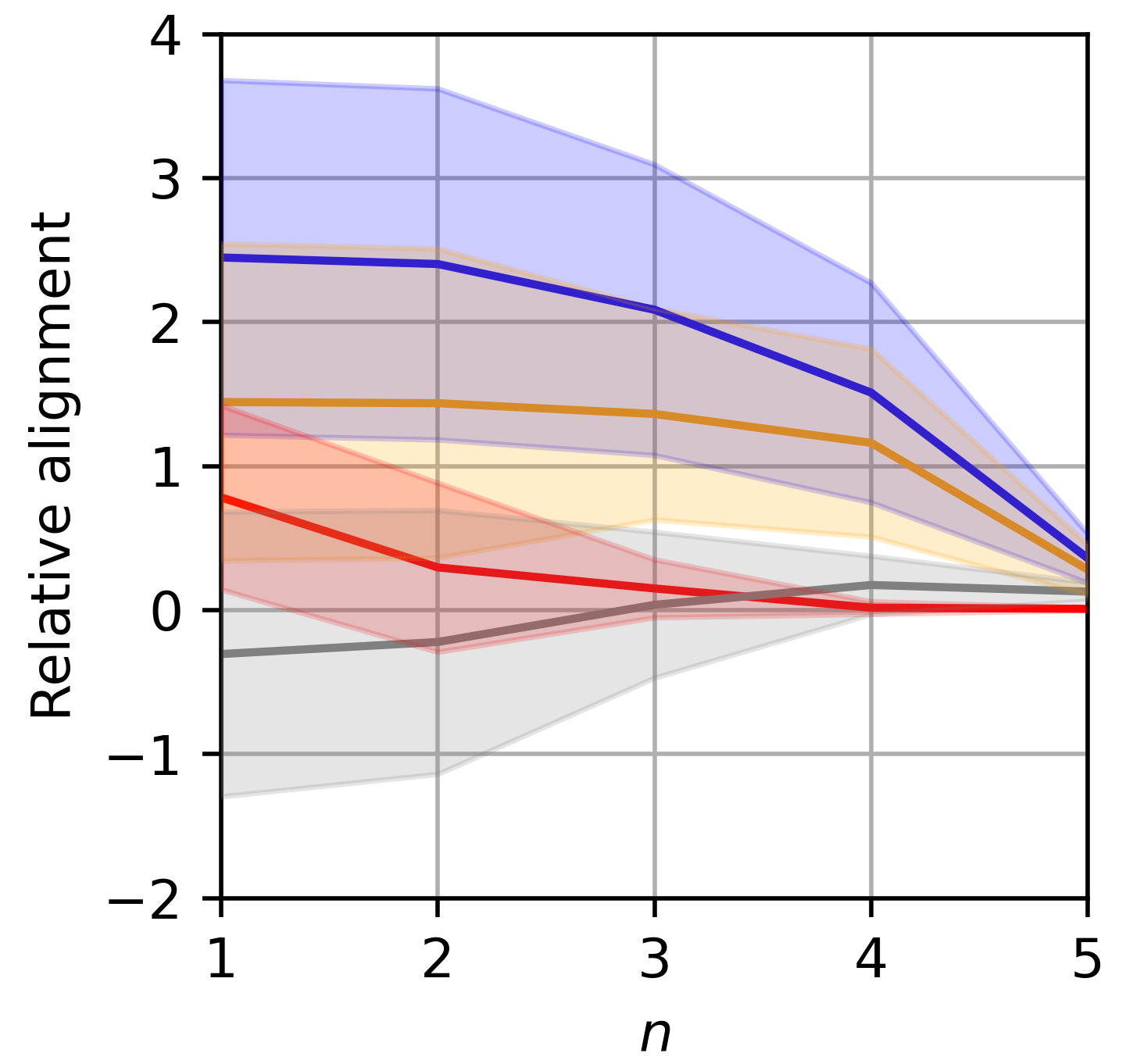}
        \caption{Relative alignment at time $t_n$.}
    \end{subfigure}
    \begin{subfigure}[b]{0.23\textwidth}
        \includegraphics[width=\textwidth]{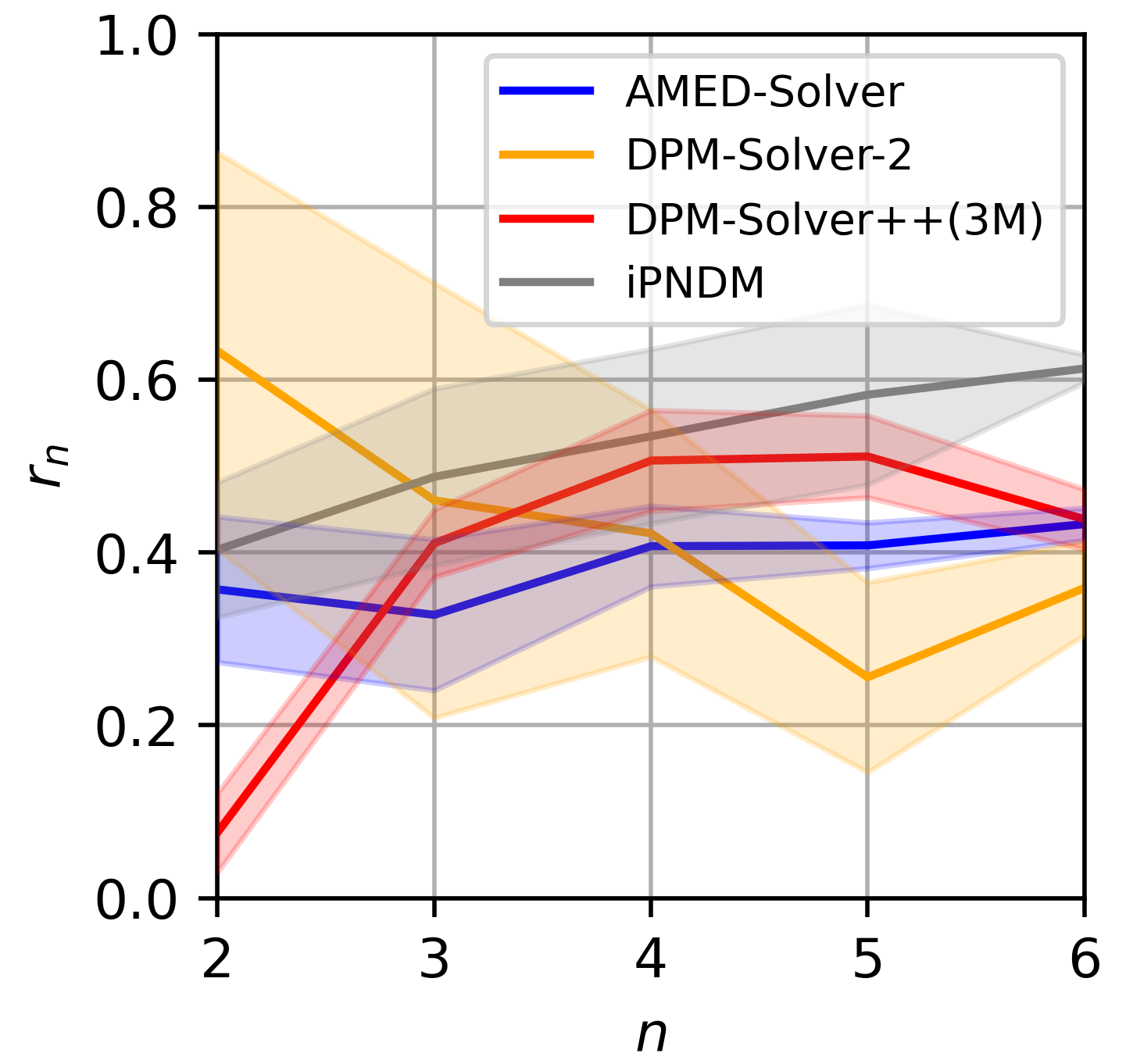}
        \caption{Best $r_n$ given by grid search.}
    \end{subfigure}
    \caption{Effectiveness of searching the intermediate time steps. Given a time schedule $\Gamma = \{t_1=\epsilon, \cdots, t_N=T\}$ where $\epsilon=0.002, T=80, N=6$, we first generate a ground truth trajectory. For each ODE solver, we generate a baseline trajectory by performing evaluations at $s_n=\sqrt{t_nt_{n+1}}$, and a searched trajectory by a greedy grid search on $r_n$ which gives $s_n=t_n^{r_n} t_{n+1}^{1-{r_n}}$. 
    }
    \label{fig:align}
\end{figure}

\subsection{Training and Sampling}
\label{subsec:training_and_sampling}

As samples from different sampling trajectories approach the asymmetric data manifold, their current locations should contribute to the corresponding trajectory curvatures~\cite{chen2023geometric}. To recognize the sample location without extra computation overhead, we extract the bottleneck feature of the pre-trained U-Net model every time after its evaluation. We then take the current and next time step $t_{n+1}$ and $t_n$ along with the bottleneck feature $\bfh_{t_{n+1}}$ as the inputs to the AMED predictor $g_\phi$ to predict the intermediate time step $s_n$ and the scaling factor $c_n$. Formally, we have
\begin{equation}
    \label{eq:network_formulation}
    \{s_n, c_n\} = g_{\phi}(\bfh_{t_{n+1}}, t_{n+1}, t_n).
\end{equation}
The network architecture is shown in \cref{fig:network}.

\begin{figure}
  \centering
    \includegraphics[width=\linewidth]{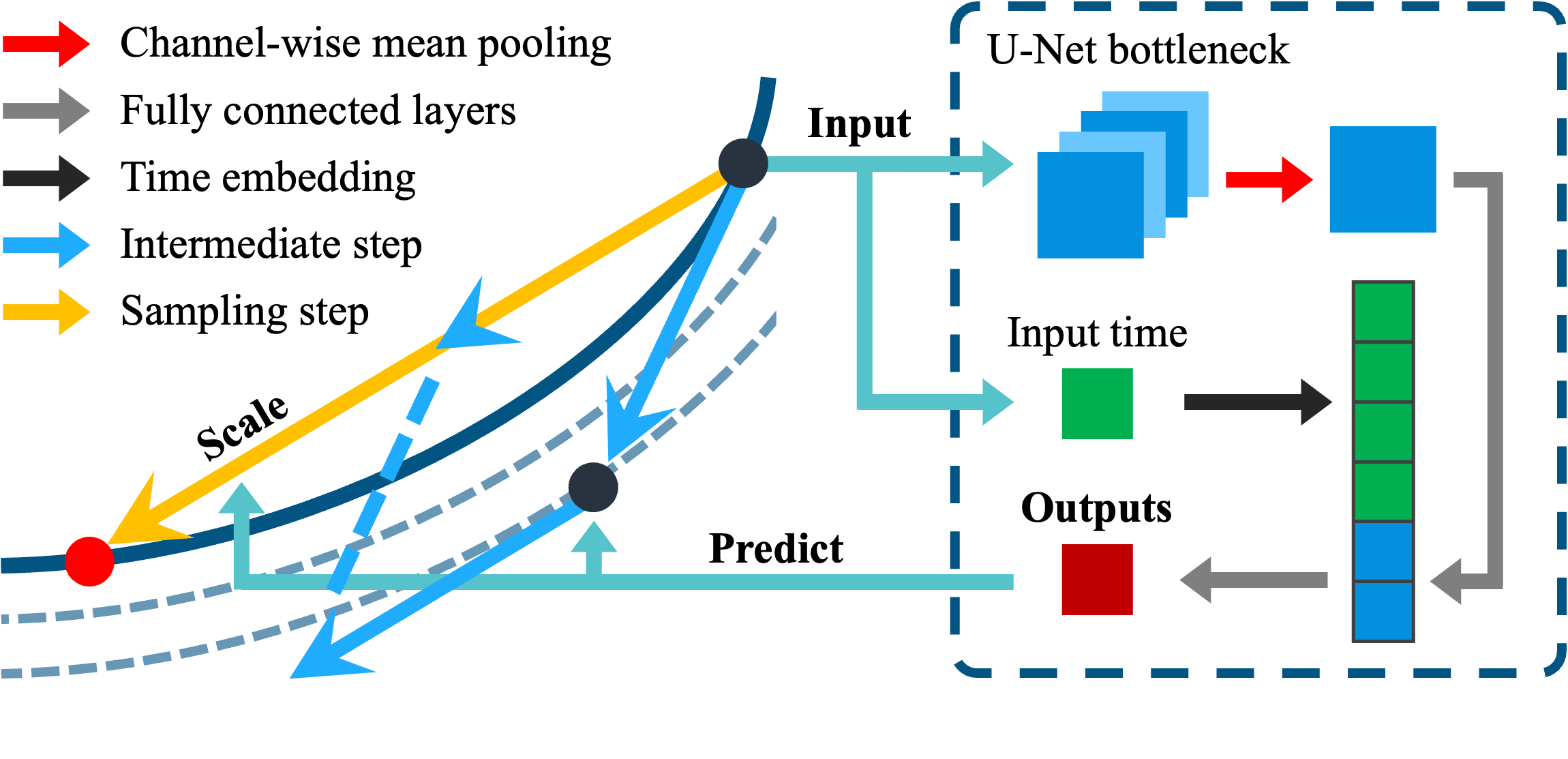}
    \caption{Network Architecture. Given the bottleneck feature extracted by the U-Net model at time $t_{n+1}$, we perform the channel-wise mean pooling and pass it through two fully connected layers. It is concatenated with the time embedding and goes through one extra fully connected layer and a sigmoid function to output $r_n$ and $c_n$. The intermediate time step is then given by $s_n=t_n^{r_n} t_{n+1}^{1-{r_n}}$.}
    \label{fig:network}
\end{figure}

As for sampling from time $t_{n+1}$ to $t_n$, we first perform one U-Net evaluation at $t_{n+1}$ and extract the bottleneck feature $\bfh_{n+1}$ to predict $s_n$ and $c_n$. For AMED-Solver, we obtain $\bfx_{s_n}$ by an Euler step from $t_{n+1}$ to $s_n$ and then use \cref{eq:approx2} to obtain $\bfx_{t_n}$. When applying AMED-Plugin on other ODE solvers, we step from $t_{n+1}$ to $s_n$ and $s_n$ to $t_n$ following the original solver's sampling procedure and $c_n$ is used to scale the direction in the latter step. 
The total NFE is thus $2(N-1)$. We denote such a sampling step from $t_{n+1}$ to $t_n$ by
\begin{equation}
    \label{eq:ode_solver}
    \bfx_{t_n} = \Phi(\bfx_{t_{n+1}}, t_{n+1}, t_{n}, \Lambda_n),
\end{equation}
where $\Lambda_n$ is the set of intermediate time steps $s_n \in (t_n, t_{n+1})$ and scaling factors $c_n$ introduced in this step.

The training of $g_\phi$ is based on knowledge distillation, where the student and teacher sampling trajectories evaluated at $\Gamma$ are required, and they are denoted as $\{\bfx_{t_n}\}_{n=1}^{N}$ and $\{\bfy_{t_n}\}_{n=1}^{N}$, respectively.
We then denote the sampling process that generates student and teacher trajectories by $\Phi_s$ with $\Lambda_{n}^s=\{s_n, c_n\}$ and $\Phi_t$ with $\Lambda_{n}^t$, respectively. 
Since teacher trajectories require more NFE to give reliable reference, we set the intermediate time steps to be an interpolation of $M$ steps between $t_n$ and $t_{n+1}$ following the original time schedule. Taking the polynomial schedule as example, we set $\Lambda_{n}^t = \{s_n^1\, \cdots, s_n^M, c_n^1=1\, \cdots, c_n^M=1\}$, where
\begin{equation}
    \label{eq:poly_schedule_teacher}
    s_n^i=(t_n^{1/\rho}+\frac{i}{M+1}(t_{n+1}^{1/\rho}-t_n^{1/\rho}))^{\rho}.
\end{equation}

We train $g_\phi$ using a distance metric $d(\cdot, \cdot)$ between samples on both trajectories with $\{s_n, c_n\}$ predicted by $g_\phi$:
\begin{equation}
    \label{eq:loss}
    \mathcal{L}_{t_n}(\phi) = d(\Phi_s(\bfx_{t_{n+1}}, t_{n+1}, t_{n}, \{s_n, c_n\}), \bfy_{t_n})
\end{equation}
In one training loop, we first generate a batch of noise images at $t_N$ and the teacher trajectories. We then calculate the loss and update $g_\phi$ progressively from $t_{N-1}$ to $t_1$. Hence, $N-1$ backpropagations are applied in one training loop. Algorithms for training and sampling is provided in \cref{alg:training} and \cref{alg:sampling}.

Similar to the previous discovery that the sampling trajectory is nearly straight when $t$ is large \cite{chen2023geometric}, we notice that the gradient term $\bfeps_{\theta}(\bfx_{t_N}, t_N)$ at time $t_N$ shares almost the same direction as $\bfx_{t_N}$. We thus simply use $\bfx_{t_N}$ as the direction in the first sampling step to save one NFE, which is important when the NFE budget is limited. This trick is called analytical first step (AFS)~\cite{dockhorn2022genie}. We find that the application of AFS yields little degradation or even increase of sample quality for datasets with small resolutions. 

Inspired by the concurrent work~\cite{xia2023towards}, when applying our AMED-Plugin on DDIM~\cite{song2021ddim}, iPNDM~\cite{zhang2023deis} and DPM-Solver++~\cite{lu2022dpmpp} on datasets with small resolution (32$\times$32 and 64$\times$64), we optionally train additional time scaling factors $\{a_n\}_{n=1}^{N-1}$ through $g_\phi$ to expand the solution space. Given $a_n$, we use $\bfeps_{\theta}(\bfx_{s_n}, a_n s_n)$ instead of $\bfeps_{\theta}(\bfx_{s_n}, s_n)$ when stepping to $t_n$, which improves results in some cases.

\begin{algorithm}[t]
\caption{Training of $g_\phi$}
\begin{algorithmic}
  \label{alg:training}
  \STATE \textbf{Input:} Model parameter $\phi$, time schedule $\{t_n\}_{n=1}^N$, ODE solver $\Phi_s$ and $\Phi_t$.
  \REPEAT
    \STATE Sample $\bfx_{t_N} = \bfy_{t_N} \sim \calN(\mathbf{0}, t_N^2\bfI)$
    \STATE Generate a teacher trajectory $\{\bfy_{t_n}\}_{n=1}^{N}$ by $\Phi_t$
    \FOR{$n=N-1$ \textbf{to} $1$}
      \STATE $\bfx_{t_n} \gets \Phi_s(\bfx_{t_{n+1}}, t_{n+1}, t_{n}, g_{\phi}(\bfh_{t_{n+1}}, t_{n+1}, t_n))$
      \STATE \# $\bfh_{t_{n+1}}$ is extracted after U-Net evaluation at $t_{n+1}$.
      \STATE $\mathcal{L}_{t_n}(\phi) = d(\bfx_{t_n}, \bfy_{t_n})$
      \STATE Update the model parameter $\phi$
    \ENDFOR
  \UNTIL convergence
\end{algorithmic}
\end{algorithm}

\begin{algorithm}[t]
\caption{AMED Sampling}
\begin{algorithmic}
  \label{alg:sampling}
  \STATE \textbf{Input:} Trained AMED predictor $g_\phi$, time schedule $\{t_n\}_{n=1}^N$, ODE solver $\Phi$.
  \STATE Sample $\bfx_{t_N} \sim \calN(\mathbf{0}, t_N^2\bfI)$
  \FOR{$n=N-1$ \textbf{to} $1$}
    \STATE $\bfx_{t_n} \gets \Phi(\bfx_{t_{n+1}}, t_{n+1}, t_{n}, g_{\phi}(\bfh_{t_{n+1}}, t_{n+1}, t_n))$
    \STATE \# $\bfh_{t_{n+1}}$ is extracted after U-Net evaluation at $t_{n+1}$.
  \ENDFOR
  \STATE \textbf{Output:} Generated sample $\bfx_{t_1}$
\end{algorithmic}
\end{algorithm}

\subsection{Comparing with Distillation-based Methods}
\label{subsec:comparison_distillation}
Though being a solver-based method, our proposed AMED-Solver shares a similar principal with distillation-based methods. The main difference is that distillation-based methods finally build a mapping from noise to data distribution by fine-tuning the pre-trained model or training a new prediction model from scratch~\cite{luhman2021knowledge,salimans2022progressive,song2023consistency,chen2023geometric}, while our AMED-Solver still follows the nature of solving an ODE, building a probability flow from noise to image. Distillation-based methods have shown impressive results of performing high-quality generation by only one NFE~\cite{song2023consistency}. However, these methods require huge efforts on training. One should carefully design the training details and it takes a large amount of time to train the model (usually several or even tens of GPU days). Moreover, as distillation-based models 
directly build the mapping like typical generative models, they suffer from the inability of interpolating between two disconnected modes~\cite{salmona2022can}. 

Additionally, 
our training goal is to properly predict several parameters in the sampling process instead of directly predicting high-dimensional samples at the next time as those distillation-based methods. Therefore, our architecture is very simple and is easy to train thanks to the geometric property of sampling trajectories. Besides, our AMED-Solver maintains the nature of ODE solver-based methods, and does not suffer from obvious internal imperfection for downstream tasks.
\section{Experiments}
\label{sec:experiment}

\renewcommand{\arraystretch}{0.83}
\begin{table*}[htbp]
  \centering
  \fontsize{8}{10}\selectfont
  \begin{subtable}[b]{0.42\textwidth}
    \centering
    \begin{tabular}{lcccc}
        \toprule
        \multirow{2}{*}{Method} & \multicolumn{4}{c}{NFE} \\
        \cmidrule{2-5}
        & 3 & 5 & 7 & 9 \\
        \midrule
        \fontsize{7}{1}\selectfont \textbf{Multi-step solvers} \\
        \midrule
        DPM-Solver++(3M)~\cite{lu2022dpmpp}         & 110.0 & 24.97 & 6.74 & 3.42 \\
        UniPC~\cite{zhao2023unipc}                  & 109.6 & 23.98 & 5.83 & 3.21 \\
        iPNDM~\cite{liu2022pseudo,zhang2023deis}    & 47.98 & 13.59 & 5.08 & 3.17 \\
        \midrule
        \fontsize{7}{1}\selectfont \textbf{Single-step solvers} \\
        \midrule
        DDIM~\cite{song2021ddim}                    & 93.36 & 49.66 & 27.93 & 18.43 \\
        EDM~\cite{karras2022edm}                    & 306.2 & 97.67 & 37.28 & 15.76 \\
        DPM-Solver-2~\cite{lu2022dpm}               & 155.7 & 57.30 & 10.20 & 4.98 \\
        \textbf{AMED-Solver (ours)}                 & 18.49 & 7.59  & 4.36  & 3.67 \\
        \midrule
        \textbf{AMED-Plugin (ours)}                 & \textbf{10.81}$\dagger$ & \textbf{6.61}$\dagger$ & \textbf{3.65}$\dagger$ & \textbf{2.63}$\dagger$ \\
        \bottomrule
    \end{tabular}
    \caption{Unconditional generation on CIFAR10 32$\times$32. We show the results of AMED-Plugin applied on iPNDM.}
    \label{subtab:fid1}
  \end{subtable}
  \quad\quad\quad\quad
  \begin{subtable}[b]{0.42\textwidth}
    \centering
    \begin{tabular}{lcccc}
        \toprule
        \multirow{2}{*}{Method} & \multicolumn{4}{c}{NFE} \\
        \cmidrule{2-5}
        & 3 & 5 & 7 & 9 \\
        \midrule
        \fontsize{7}{1}\selectfont \textbf{Multi-step solvers} \\
        \midrule
        DPM-Solver++(3M)~\cite{lu2022dpmpp}         & 91.52 & 25.49 & 10.14 & 6.48 \\
        UniPC~\cite{zhao2023unipc}                  & 91.38 & 24.36 & 9.57  & 6.34 \\
        iPNDM~\cite{liu2022pseudo,zhang2023deis}    & 58.53 & 18.99 & 9.17  & 5.91 \\
        \midrule
        \fontsize{7}{1}\selectfont \textbf{Single-step solvers} \\
        \midrule
        DDIM~\cite{song2021ddim}                    & 82.96 & 43.81 & 27.46 & 19.27 \\
        EDM~\cite{karras2022edm}                    & 249.4 & 89.63 & 37.65 & 16.76 \\
        DPM-Solver-2~\cite{lu2022dpm}               & 140.2 & 42.41 & 12.03 & 6.64  \\
        \textbf{AMED-Solver (ours)}                 & 38.10 & \textbf{10.74} & \textbf{6.66} & \textbf{5.44} \\
        \midrule
        \textbf{AMED-Plugin (ours)}                 & \textbf{28.06} & 13.83 & 7.81 & 5.60 \\
        \bottomrule
    \end{tabular}
    \caption{Conditional generation on ImageNet 64$\times$64. We show the results of AMED-Plugin applied on iPNDM.}
    \label{subtab:fid2}
  \end{subtable}

  \vspace{0.3cm}
  \centering
  \fontsize{8}{10}\selectfont
  \begin{subtable}[b]{0.42\textwidth}
    \centering
    \begin{tabular}{lcccc}
      \toprule
      \multirow{2}{*}{Method} & \multicolumn{4}{c}{NFE} \\
      \cmidrule{2-5}
      & 3 & 5 & 7 & 9 \\
      \midrule
      \fontsize{7}{1}\selectfont \textbf{Multi-step solvers} \\
      \midrule
      DPM-Solver++(3M)~\cite{lu2022dpmpp}       & 86.45 & 22.51 & 8.44 & 4.77 \\
      UniPC~\cite{zhao2023unipc}                & 86.43 & 21.40 & 7.44 & 4.47 \\
      iPNDM~\cite{liu2022pseudo,zhang2023deis}  & 45.98 & 17.17 & 7.79 & 4.58 \\
      \midrule
      \fontsize{7}{1}\selectfont \textbf{Single-step solvers} \\
      \midrule
      DDIM~\cite{song2021ddim}                  & 78.21 & 43.93 & 28.86 & 21.01 \\
      DPM-Solver-2~\cite{lu2022dpm}             & 266.0 & 87.10 & 22.59 & 9.26 \\
      \textbf{AMED-Solver (ours)}               & 47.31 & 14.80 & 8.82 & 6.31 \\
      \midrule
      \textbf{AMED-Plugin (ours)}               & \textbf{26.87} & \textbf{12.49} & \textbf{6.64} & \textbf{4.24}$\dagger$ \\
      \bottomrule
    \end{tabular}
    \caption{Unconditional generation on FFHQ 64$\times$64. We show the results of AMED-Plugin applied on iPNDM.}
    \label{subtab:fid3}
  \end{subtable}
  \quad\quad\quad\quad
  \begin{subtable}[b]{0.42\textwidth}
    \centering
    \begin{tabular}{lcccc}
      \toprule
      \multirow{2}{*}{Method} & \multicolumn{4}{c}{NFE} \\
      \cmidrule{2-5}
      & 3 & 5 & 7 & 9 \\
      \midrule
      \fontsize{7}{1}\selectfont \textbf{Multi-step solvers} \\
      \midrule
      DPM-Solver++(3M)~\cite{lu2022dpmpp}       & 111.9 & 23.15 & 8.87 & 6.45 \\
      UniPC~\cite{zhao2023unipc}                & 112.3 & 23.34 & 8.73 & 6.61 \\
      iPNDM~\cite{liu2022pseudo,zhang2023deis}  & 80.99 & 26.65 & 13.80 & 8.38 \\
      \midrule
      \fontsize{7}{1}\selectfont \textbf{Single-step solvers} \\
      \midrule
      DDIM~\cite{song2021ddim}                  & 86.13 & 34.34 & 19.50 & 13.26 \\
      DPM-Solver-2~\cite{lu2022dpm}             & 210.6 & 80.60 & 23.25 & 9.61 \\
      \textbf{AMED-Solver (ours)}               & \textbf{58.21} & \textbf{13.20} & \textbf{7.10} & \textbf{5.65} \\
      \midrule
      \textbf{AMED-Plugin (ours)}               & 101.5 & 25.68 & 8.63 & 7.82 \\
      \bottomrule
    \end{tabular}
    \caption{Unconditional generation on LSUN Bedroom 256$\times$256. We show the results of AMED-Plugin applied on DPM-Solver-2.}
    \label{subtab:fid4}
  \end{subtable}
  \caption{Results of image generation. Our proposed AMED-Solver and AMED-Plugin achieve state-of-the-art results among solver-based methods in around 5 NFE. $\dagger$: additional time scaling factors $\{a_n\}_{n=1}^{N-1}$ are trained.}
  \label{tab:fid}
\end{table*}

\subsection{Settings}
\label{subsec:settings}
\noindent
\textbf{Datasets.} We employ AMED-Solver and AMED-Plugin on a wide range of datasets with image resolutions ranging from 32 to 512, including CIFAR10 32$\times$32~\cite{krizhevsky2009learning}, FFHQ 64$\times$64~\cite{karras2019style}, ImageNet 64$\times$64~\cite{russakovsky2015ImageNet}, LSUN Bedroom 256$\times$256~\cite{yu2015lsun}. We also give quantitative and qualitative results on stable-diffusion~\cite{rombach2022ldm} with resolution of 512.

\noindent
\textbf{Models.} The pre-trained models are pixel-space models from~\cite{karras2022edm} and ~\cite{song2023consistency} and latent-space model from~\cite{rombach2022ldm}.

\noindent
\textbf{Solvers.} We reimplement several representative fast ODE solvers including DDIM~\cite{song2021ddim}, DPM-Solver-2~\cite{lu2022dpm}, multi-step DPM-Solver++~\cite{lu2022dpmpp}, UniPC~\cite{zhao2023unipc} and improved PNDM (iPNDM)~\cite{liu2022pseudo,zhang2023deis}. It is worth mentioning that we find iPNDM achieves very impressive results and outperforms other ODE solvers in many cases.

\noindent
\textbf{Time schedule.} We mainly use the polynomial time schedule with $\rho=7$, which is the default setting in~\cite{karras2022edm}, except for DPM-Solver++ and UniPC where we use logSNR schedule recommended in original papers~\cite{lu2022dpmpp,zhao2023unipc} for better results. Besides, for AMED-Solver on CIFAR10 32$\times$32~\cite{krizhevsky2009learning}, FFHQ 64$\times$64~\cite{karras2019style} and ImageNet 64$\times$64~\cite{russakovsky2015ImageNet}, we use uniform time schedule which is widely used in papers with a DDPM backbone~\cite{ho2020ddpm}.

\noindent
\textbf{Training.} The total parameter of the AMED predictor $g_\phi$ is merely 9k. We train $g_\phi$ for 10k images, which takes 2-8 minutes on CIFAR10 and 1-3 hours on LSUN Bedroom using a single NVIDIA A100 GPU. L2 norm is used as the distance metric in \cref{eq:loss} for all experiments. 
We use DPM-Solver-2~\cite{lu2022dpm} or EDM~\cite{karras2022edm} with doubled NFE ($M = 1$) to generate teacher trajectories for AMED-Solver, while using the same solver that generates student trajectories for AMED-Plugin, with $M = 1$ for DPM-Solver-2 and $M = 2$ else. Detailed discussion is provided in \cref{subsec:sup_ablation_steps}.

\noindent
\textbf{Sampling.} Our AMED-Solver and AMED-Plugin naturally create solvers with even NFE. Once AFS is used, the total NFE becomes odd. With the goal of designing fast solvers with the small NFE, we mainly test our method on NFE $\in \{3, 5, 7, 9\}$ where AFS is applied. More results on NFE $\in \{4, 6, 8, 10\}$ without AFS is provided in Appendix~\ref{subsec:sup_ablation_afs}.

\noindent
\textbf{Evaluation.} We measure the sample quality via Fr\'{e}chet Inception Distance (FID)~\cite{heusel2017gans} with 50k images.
For Stable-Diffusion, we follow~\cite{ramesh2021zero} and evaluate the FID value by 30k images generated by 30k fixed prompts sampled from the MS-COCO~\cite{lin2014microsoft} validation set.

\subsection{Image Generation}
\label{subsec:image_generation}

In this section, we show the results of image generation.
For datasets with small resolution such as CIFAR10 32$\times$32, FFHQ 64$\times$64 and ImageNet 64$\times$64, we report the results of AMED-Plugin on iPNDM solver for its leading results. For large-resolution datasets like LSUN Bedroom, we report the results of AMED-Plugin on DPM-Solver-2 since the use of AFS causes inferior results for multi-step solvers in this case (see \cref{subsec:sup_ablation_afs} for a detailed discussion). We implement DPM-Solver++ and UniPC with order of 3, iPNDM with order of 4. To report results of DPM-Solver-2 and EDM with odd NFE, we apply AFS in their first steps.

\renewcommand{\arraystretch}{0.8}
\begin{table}[tbp]
  \centering
  \fontsize{8}{10}\selectfont
    \centering
    \begin{tabular}{lcccc}
      \toprule
      \multirow{2}{*}{Method} & \multicolumn{4}{c}{NFE (1 step = 2 NFE)} \\
      \cmidrule{2-5}
      & 8 & 12 & 16 & 20 \\
      \midrule
      DPM-Solver++(2M)~\cite{lu2022dpmpp}       & 21.33 & 15.99 & 14.84 & 14.58 \\
      AMED-Plugin (\textbf{ours})               & \textbf{18.92} & \textbf{14.84} & \textbf{13.96} & \textbf{13.24} \\
      \bottomrule
    \end{tabular}
  \caption{FID results on Stable-Diffusion~\cite{rombach2022ldm}. The AMED-Plugin is applied on DPM-Solver++(2M).}
  \label{tab:sup_fid_stable_diff}
\end{table}

The results are shown in \cref{tab:fid}. Our AMED-Solver outperforms other single-step methods and can even beat multi-step methods in many cases. 
For the AMED-Plugin, we find it showing large boost when applied on various solvers especially for DPM-Solver-2 as shown in \cref{subtab:fid4}. Notably, the AMED-Plugin on iPNDM improves the FID by 6.98, 4.68 and 5.16 on CIFAR10 32$\times$32, ImageNet 64$\times$64 and FFHQ 64$\times$64 in 5 NFE. Our methods achieve state-of-the-art results in solver-based methods in around 5 NFE. 

As for Stable-Diffusion~\cite{rombach2022ldm}, we use the v1.4 checkpoint with default guidance scale 7.5. Samples are generated by DPM-Solver++(2M) as recommended in the official implementation. The quantitative results are shown in \cref{tab:sup_fid_stable_diff}.

In \cref{fig:learned_results}, we show the learned parameters of $g_\phi$ for AMED-Solver, where $r_n$ and $c_n$ are predicted by $g_\phi$ and $s_n=t_n^{r_n} t_{n+1}^{1-r_n}$. The dashed line denotes the default setting of DPM-Solver-2. 
We provide more quantitative as well as qualitative results in \cref{sec:sup_additional_results}.

\begin{figure}[tbp]
    \centering
    \includegraphics[width=\linewidth]{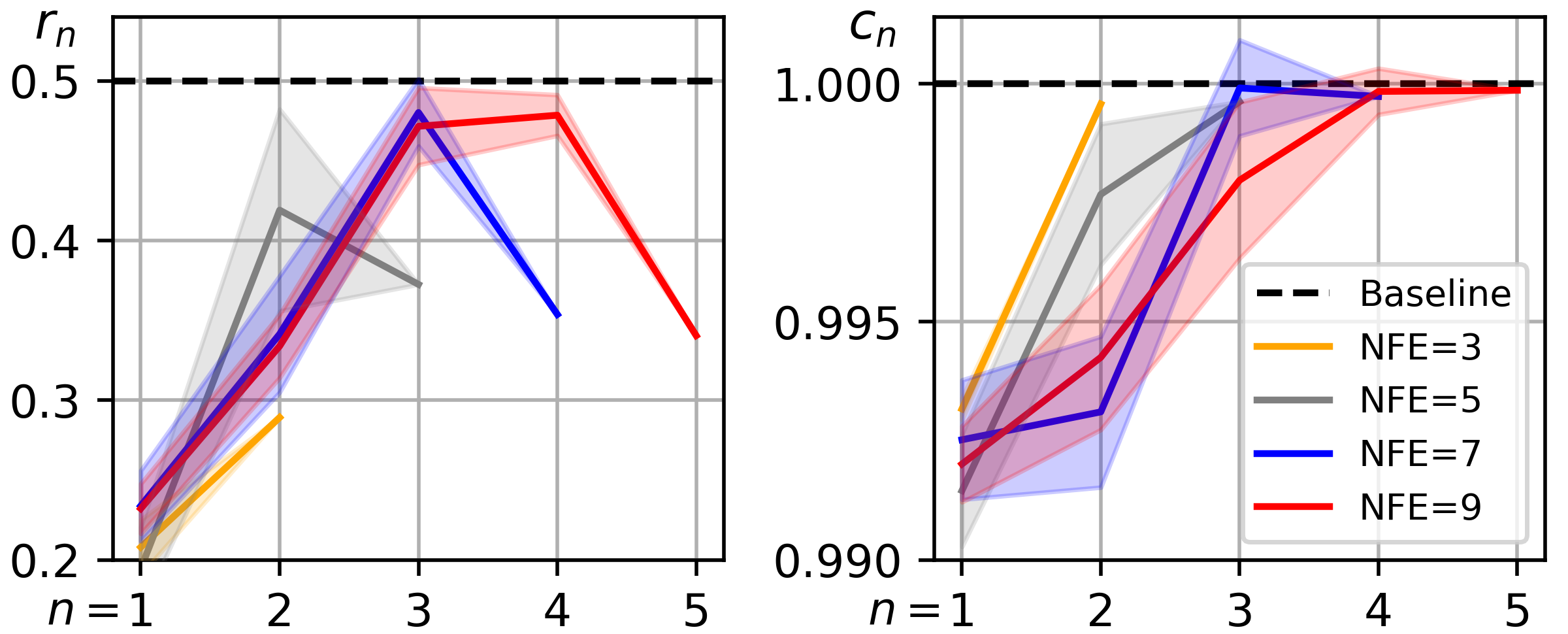}
    \caption{The learned coefficients $r_n$ and $c_n$ vary in different steps and the mean values are consistently lower than the default setting.}
    \label{fig:learned_results}
\end{figure}

\subsection{Ablation Study}
\label{subsec:ablation}

\noindent
\textbf{Teacher solvers.} 
In \cref{tab:fid_ablation_teacher}, we test AMED-Solver and AMED-Plugin on iPNDM with different teacher solvers for the training of $g_\phi$ on CIFAR10. It turns out that the best results are achieved when teacher solver resembles the student solver in the sampling process.

\renewcommand{\arraystretch}{0.8}
\begin{table}[tbp]
    \centering
    \fontsize{8}{10}\selectfont
    \begin{tabular}{lcccc}
      \toprule
      \multirow{2}{*}{Teacher Solver} & \multicolumn{4}{c}{NFE} \\
      \cmidrule{2-5}
      & 3 & 5 & 7 & 9 \\
      \midrule
      \fontsize{7}{1}\selectfont \textbf{AMED-Solver} \\
      \midrule
        DPM-Solver++(3M)~\cite{lu2022dpmpp}         & 35.68 & 12.34 & 11.28 & 9.65 \\
        iPNDM~\cite{liu2022pseudo,zhang2023deis}    & 32.38 & 12.42 & 10.09 & 8.54 \\
        DPM-Solver-2~\cite{lu2022dpm}               & \textbf{18.49} & 11.60 & 11.64 & 9.23 \\
        EDM~\cite{karras2022edm}                    & 28.99 & \textbf{7.59} & \textbf{4.36} & \textbf{3.67} \\
      \midrule
      \fontsize{7}{1}\selectfont \textbf{AMED-Plugin on iPNDM} \\
      \midrule
        DPM-Solver++(3M)~\cite{lu2022dpmpp}         & 16.00 & 7.57 & 4.28 & 2.85 \\
        iPNDM~\cite{liu2022pseudo,zhang2023deis}    & \textbf{10.81} & \textbf{6.61} & \textbf{3.65} & \textbf{2.63} \\
        DPM-Solver-2~\cite{lu2022dpm}               & 12.07 & 8.19 & 4.52 & 2.66 \\
        EDM~\cite{karras2022edm}                    & 29.62 & 10.58 & 9.36 & 4.44 \\
      \bottomrule
    \end{tabular}
    \caption{Comparison of teacher solvers on CIFAR10.}
    \label{tab:fid_ablation_teacher}
\end{table}

\noindent
\textbf{Time schedule.} We observe that different fast ODE solvers have different preference on time schedules. This preference even depends on the used dataset. In \cref{tab:limitation}, we provide results for DPM-Solver++(3M)~\cite{lu2022dpmpp} on CIFAR10 with different time schedules and find that this solver prefers logSNR schedule where the interval between the first and second time step is larger than other tested schedules. 
As a comparison, our AMED-Plugin applied on these cases largely and consistently improves the results, irrespective of the specific time schedule.

\renewcommand{\arraystretch}{0.8}
\begin{table}[htbp]
    \centering
    \fontsize{8}{11}\selectfont
    \begin{tabular}{lcccc}
      \toprule
      \multirow{2}{*}{Time schedule} & \multicolumn{4}{c}{NFE} \\
      \cmidrule{2-5}
      & 3 & 5 & 7 & 9 \\
      \midrule
      \multicolumn{5}{l}{\fontsize{7}{1}\selectfont \textbf{DPM-Solver++(3M)}} \\
      \midrule
        Uniform             & 76.80 & 26.90 & 16.80 & 13.44 \\
        Polynomial          & 70.04 & 31.66 & 11.30 & 6.45 \\
        logSNR              & 110.0 & 24.97 & 6.74  & 3.42 \\
      \midrule
      \multicolumn{5}{l}{\fontsize{7}{1}\selectfont \textbf{AMED-Plugin on DPM-Solver++(3M)}} \\
      \midrule
        Uniform             & 33.61 & 13.24 & 8.89 & 8.24 \\
        Polynomial          & 32.47 & 19.59 & 9.60 & 4.39 \\
        logSNR              & \textbf{25.95} & \textbf{7.68} & \textbf{4.51} & \textbf{3.03} \\
      \bottomrule
    \end{tabular}
    \caption{FID results under different time schedules on CIFAR10. The DPM-Solver++(3M) shows preference on the uniform logSNR schedule given larger NFE. Our AMED-Plugin consistently improves the results.}
    \label{tab:limitation}
\end{table}

\section{Conclusion}
\label{sec:conclusion}
In this paper, we introduce a single-step ODE solver called AMED-Solver to minimize the discretization error for fast diffusion sampling. Our key observation is that each sampling trajectory generated by existing ODE solvers approximately lies in a two-dimensional subspace, and thus the mean value theorem guarantees the learning of an approximate mean direction. The AMED-Solver effectively mitigates the problem of rapid sample quality degradation, which is commonly encountered in single-step methods, and shows promising results on datasets with large resolution. We also generalize the idea of AMED-Solver to AMED-Plugin, a plugin that can be applied on any fast ODE solvers to further improve the sample quality. We validate our methods through extensive experiments, and our methods achieve state-of-the-art results in extremely small NFE (around 5). We hope our attempt could inspire future works to further release the potential of fast solvers for diffusion models. 

\noindent
\textbf{Limitation and Future Work.}
Fast ODE solvers for diffusion models are highly sensitive to time schedules especially when the NFE budget is limited (See \cref{tab:limitation}). Through experiments, we observe that any fixed time schedule fails to perform well in all situations.
In fact, our AMED-Plugin can be treated as adjusting half of the time schedule to partly alleviates but not avoids this issue. 
We believe that a better designation of time schedule requires further knowledge to the geometric shape of the sampling trajectory~\cite{chen2023geometric}. We leave this to the future work. 

\section{Acknowledgement}
\label{sec:acknowledgement}
The authors would like to thank Jianlin Su and anonymous reviewers for their helpful comments. This work is supported by the Starry Night Science Fund of Zhejiang University Shanghai Institute for Advanced Study, China (Grant No: SN-ZJU-SIAS-001), National Natural Science Foundation of China (Grant No: U1866602) and the advanced computing resources provided by the Supercomputing Center of Hangzhou City University. 
\clearpage
{
    \small
    \bibliographystyle{ieeenat_fullname}
    \bibliography{main}

\begin{thebibliography}{52}
\providecommand{\natexlab}[1]{#1}
\providecommand{\url}[1]{\texttt{#1}}
\expandafter\ifx\csname urlstyle\endcsname\relax
  \providecommand{\doi}[1]{doi: #1}\else
  \providecommand{\doi}{doi: \begingroup \urlstyle{rm}\Url}\fi

\bibitem[Anderson(1982)]{anderson1982reverse}
Brian~DO Anderson.
\newblock Reverse-time diffusion equation models.
\newblock \emph{Stochastic Processes and their Applications}, 12\penalty0 (3):\penalty0 313--326, 1982.

\bibitem[Atkinson et~al.(2011)Atkinson, Han, and Stewart]{atkinson2011numerical}
Kendall Atkinson, Weimin Han, and David~E Stewart.
\newblock \emph{Numerical solution of ordinary differential equations}.
\newblock John Wiley \& Sons, 2011.

\bibitem[Bao et~al.(2022)Bao, Li, Zhu, and Zhang]{bao2022analytic}
Fan Bao, Chongxuan Li, Jun Zhu, and Bo Zhang.
\newblock Analytic-dpm: an analytic estimate of the optimal reverse variance in diffusion probabilistic models.
\newblock \emph{arXiv preprint arXiv:2201.06503}, 2022.

\bibitem[Berthelot et~al.(2023)Berthelot, Autef, Lin, Yap, Zhai, Hu, Zheng, Talbot, and Gu]{berthelot2023tract}
David Berthelot, Arnaud Autef, Jierui Lin, Dian~Ang Yap, Shuangfei Zhai, Siyuan Hu, Daniel Zheng, Walter Talbot, and Eric Gu.
\newblock Tract: Denoising diffusion models with transitive closure time-distillation.
\newblock \emph{arXiv preprint arXiv:2303.04248}, 2023.

\bibitem[Bi{\'n}kowski et~al.(2018)Bi{\'n}kowski, Sutherland, Arbel, and Gretton]{binkowski2018demystifying}
Miko{\l}aj Bi{\'n}kowski, Danica~J Sutherland, Michael Arbel, and Arthur Gretton.
\newblock Demystifying mmd gans.
\newblock \emph{arXiv preprint arXiv:1801.01401}, 2018.

\bibitem[Chen et~al.(2023)Chen, Zhou, Mei, Shen, Chen, and Wang]{chen2023geometric}
Defang Chen, Zhenyu Zhou, Jian-Ping Mei, Chunhua Shen, Chun Chen, and Can Wang.
\newblock A geometric perspective on diffusion models.
\newblock \emph{arXiv preprint arXiv:2305.19947}, 2023.

\bibitem[Cheney et~al.(2001)Cheney, Cheney, and Cheney]{cheney2001analysis}
Elliott~Ward Cheney, EW Cheney, and W Cheney.
\newblock \emph{Analysis for applied mathematics}.
\newblock Springer, 2001.

\bibitem[Daras et~al.(2023)Daras, Dagan, Dimakis, and Daskalakis]{daras2023consistent}
Giannis Daras, Yuval Dagan, Alexandros~G Dimakis, and Constantinos Daskalakis.
\newblock Consistent diffusion models: Mitigating sampling drift by learning to be consistent.
\newblock \emph{arXiv preprint arXiv:2302.09057}, 2023.

\bibitem[Dhariwal and Nichol(2021)]{dhariwal2021diffusion}
Prafulla Dhariwal and Alex Nichol.
\newblock Diffusion models beat gans on image synthesis.
\newblock In \emph{Advances in Neural Information Processing Systems}, 2021.

\bibitem[Dockhorn et~al.(2022)Dockhorn, Vahdat, and Kreis]{dockhorn2022genie}
Tim Dockhorn, Arash Vahdat, and Karsten Kreis.
\newblock Genie: Higher-order denoising diffusion solvers.
\newblock In \emph{Advances in Neural Information Processing Systems}, 2022.

\bibitem[Feller(1949)]{feller1949theory}
William Feller.
\newblock On the theory of stochastic processes, with particular reference to applications.
\newblock In \emph{Proceedings of the First Berkeley Symposium on Mathematical Statistics and Probability}, pages 403--432, 1949.

\bibitem[Goodfellow et~al.(2014)Goodfellow, Pouget-Abadie, Mirza, Xu, Warde-Farley, Ozair, Courville, and Bengio]{goodfellow2014generative}
Ian Goodfellow, Jean Pouget-Abadie, Mehdi Mirza, Bing Xu, David Warde-Farley, Sherjil Ozair, Aaron Courville, and Yoshua Bengio.
\newblock Generative adversarial nets.
\newblock \emph{Advances in neural information processing systems}, 27, 2014.

\bibitem[Gu et~al.(2023)Gu, Zhai, Zhang, Liu, and Susskind]{gu2023boot}
Jiatao Gu, Shuangfei Zhai, Yizhe Zhang, Lingjie Liu, and Josh Susskind.
\newblock Boot: Data-free distillation of denoising diffusion models with bootstrapping.
\newblock \emph{arXiv preprint arXiv:2306.05544}, 2023.

\bibitem[Heusel et~al.(2017)Heusel, Ramsauer, Unterthiner, Nessler, and Hochreiter]{heusel2017gans}
Martin Heusel, Hubert Ramsauer, Thomas Unterthiner, Bernhard Nessler, and Sepp Hochreiter.
\newblock {GANs} trained by a two time-scale update rule converge to a local {Nash} equilibrium.
\newblock In \emph{Advances in Neural Information Processing Systems}, pages 6626--6637, 2017.

\bibitem[Ho et~al.(2020)Ho, Jain, and Abbeel]{ho2020ddpm}
Jonathan Ho, Ajay Jain, and Pieter Abbeel.
\newblock Denoising diffusion probabilistic models.
\newblock In \emph{Advances in Neural Information Processing Systems}, 2020.

\bibitem[Hyv{\"{a}}rinen(2005)]{hyvarinen2005estimation}
Aapo Hyv{\"{a}}rinen.
\newblock Estimation of non-normalized statistical models by score matching.
\newblock \emph{Journal of Machine Learning Research}, 6:\penalty0 695--709, 2005.

\bibitem[Karras et~al.(2019)Karras, Laine, and Aila]{karras2019style}
Tero Karras, Samuli Laine, and Timo Aila.
\newblock A style-based generator architecture for generative adversarial networks.
\newblock In \emph{Proceedings of the IEEE/CVF conference on computer vision and pattern recognition}, pages 4401--4410, 2019.

\bibitem[Karras et~al.(2022)Karras, Aittala, Aila, and Laine]{karras2022edm}
Tero Karras, Miika Aittala, Timo Aila, and Samuli Laine.
\newblock Elucidating the design space of diffusion-based generative models.
\newblock In \emph{Advances in Neural Information Processing Systems}, 2022.

\bibitem[Kingma and Welling(2013)]{kingma2013auto}
Diederik~P Kingma and Max Welling.
\newblock Auto-encoding variational bayes.
\newblock \emph{arXiv preprint arXiv:1312.6114}, 2013.

\bibitem[Kingma et~al.(2021)Kingma, Salimans, Poole, and Ho]{kingma2021vdm}
Diederik~P Kingma, Tim Salimans, Ben Poole, and Jonathan Ho.
\newblock Variational diffusion models.
\newblock In \emph{Advances in Neural Information Processing Systems}, 2021.

\bibitem[Krizhevsky and Hinton(2009)]{krizhevsky2009learning}
Alex Krizhevsky and Geoffrey Hinton.
\newblock Learning multiple layers of features from tiny images.
\newblock \emph{Technical Report}, 2009.

\bibitem[Lin et~al.(2014)Lin, Maire, Belongie, Hays, Perona, Ramanan, Doll{\'a}r, and Zitnick]{lin2014microsoft}
Tsung-Yi Lin, Michael Maire, Serge Belongie, James Hays, Pietro Perona, Deva Ramanan, Piotr Doll{\'a}r, and C~Lawrence Zitnick.
\newblock Microsoft coco: Common objects in context.
\newblock In \emph{Computer Vision--ECCV 2014: 13th European Conference, Zurich, Switzerland, September 6-12, 2014, Proceedings, Part V 13}, pages 740--755. Springer, 2014.

\bibitem[Liu et~al.(2022{\natexlab{a}})Liu, Ren, Lin, and Zhao]{liu2022pseudo}
Luping Liu, Yi Ren, Zhijie Lin, and Zhou Zhao.
\newblock Pseudo numerical methods for diffusion models on manifolds.
\newblock In \emph{International Conference on Learning Representations}, 2022{\natexlab{a}}.

\bibitem[Liu et~al.(2022{\natexlab{b}})Liu, Gong, and Liu]{liu2022flow}
Xingchao Liu, Chengyue Gong, and Qiang Liu.
\newblock Flow straight and fast: Learning to generate and transfer data with rectified flow.
\newblock \emph{arXiv preprint arXiv:2209.03003}, 2022{\natexlab{b}}.

\bibitem[Lu et~al.(2022{\natexlab{a}})Lu, Zhou, Bao, Chen, Li, and Zhu]{lu2022dpm}
Cheng Lu, Yuhao Zhou, Fan Bao, Jianfei Chen, Chongxuan Li, and Jun Zhu.
\newblock Dpm-solver: A fast ode solver for diffusion probabilistic model sampling in around 10 steps.
\newblock In \emph{Advances in Neural Information Processing Systems}, 2022{\natexlab{a}}.

\bibitem[Lu et~al.(2022{\natexlab{b}})Lu, Zhou, Bao, Chen, Li, and Zhu]{lu2022dpmpp}
Cheng Lu, Yuhao Zhou, Fan Bao, Jianfei Chen, Chongxuan Li, and Jun Zhu.
\newblock Dpm-solver++: Fast solver for guided sampling of diffusion probabilistic models.
\newblock \emph{arXiv preprint arXiv:2211.01095}, 2022{\natexlab{b}}.

\bibitem[Luhman and Luhman(2021)]{luhman2021knowledge}
Eric Luhman and Troy Luhman.
\newblock Knowledge distillation in iterative generative models for improved sampling speed.
\newblock \emph{arXiv preprint arXiv:2101.02388}, 2021.

\bibitem[Lyu(2009)]{lyu2009interpretation}
Siwei Lyu.
\newblock Interpretation and generalization of score matching.
\newblock In \emph{Proceedings of the Twenty-Fifth Conference on Uncertainty in Artificial Intelligence}, pages 359--366, 2009.

\bibitem[Maoutsa et~al.(2020)Maoutsa, Reich, and Opper]{maoutsa2020interacting}
Dimitra Maoutsa, Sebastian Reich, and Manfred Opper.
\newblock Interacting particle solutions of fokker--planck equations through gradient--log--density estimation.
\newblock \emph{arXiv preprint arXiv:2006.00702}, 2020.

\bibitem[Nichol and Dhariwal(2021)]{nichol2021improved}
Alexander~Quinn Nichol and Prafulla Dhariwal.
\newblock Improved denoising diffusion probabilistic models.
\newblock In \emph{International Conference on Machine Learning}, pages 8162--8171. PMLR, 2021.

\bibitem[Oksendal(2013)]{oksendal2013stochastic}
Bernt Oksendal.
\newblock \emph{Stochastic differential equations: an introduction with applications}.
\newblock Springer Science \& Business Media, 2013.

\bibitem[Platen and Bruti-Liberati(2010)]{platen2010numerical}
Eckhard Platen and Nicola Bruti-Liberati.
\newblock \emph{Numerical solution of stochastic differential equations with jumps in finance}.
\newblock Springer Science \& Business Media, 2010.

\bibitem[Ramesh et~al.(2021)Ramesh, Pavlov, Goh, Gray, Voss, Radford, Chen, and Sutskever]{ramesh2021zero}
Aditya Ramesh, Mikhail Pavlov, Gabriel Goh, Scott Gray, Chelsea Voss, Alec Radford, Mark Chen, and Ilya Sutskever.
\newblock Zero-shot text-to-image generation.
\newblock In \emph{International conference on machine learning}, pages 8821--8831. Pmlr, 2021.

\bibitem[Rombach et~al.(2022)Rombach, Blattmann, Lorenz, Esser, and Ommer]{rombach2022ldm}
Robin Rombach, Andreas Blattmann, Dominik Lorenz, Patrick Esser, and Bj{\"o}rn Ommer.
\newblock High-resolution image synthesis with latent diffusion models.
\newblock In \emph{Proceedings of the IEEE/CVF Conference on Computer Vision and Pattern Recognition}, pages 10684--10695, 2022.

\bibitem[Ronneberger et~al.(2015)Ronneberger, Fischer, and Brox]{ronneberger2015u}
Olaf Ronneberger, Philipp Fischer, and Thomas Brox.
\newblock U-net: Convolutional networks for biomedical image segmentation.
\newblock In \emph{Medical Image Computing and Computer-Assisted Intervention--MICCAI 2015: 18th International Conference, Munich, Germany, October 5-9, 2015, Proceedings, Part III 18}, pages 234--241. Springer, 2015.

\bibitem[Ruiz et~al.(2023)Ruiz, Li, Jampani, Pritch, Rubinstein, and Aberman]{ruiz2023dreambooth}
Nataniel Ruiz, Yuanzhen Li, Varun Jampani, Yael Pritch, Michael Rubinstein, and Kfir Aberman.
\newblock Dreambooth: Fine tuning text-to-image diffusion models for subject-driven generation.
\newblock In \emph{Proceedings of the IEEE/CVF Conference on Computer Vision and Pattern Recognition}, pages 22500--22510, 2023.

\bibitem[Russakovsky et~al.(2015)Russakovsky, Deng, Su, Krause, Satheesh, Ma, Huang, Karpathy, Khosla, Bernstein, Berg, and Li]{russakovsky2015ImageNet}
Olga Russakovsky, Jia Deng, Hao Su, Jonathan Krause, Sanjeev Satheesh, Sean Ma, Zhiheng Huang, Andrej Karpathy, Aditya Khosla, Michael~S. Bernstein, Alexander~C. Berg, and Fei{-}Fei Li.
\newblock Imagenet large scale visual recognition challenge.
\newblock \emph{International Journal of Computer Vision}, 115\penalty0 (3):\penalty0 211--252, 2015.

\bibitem[Saharia et~al.(2022)Saharia, Chan, Saxena, Li, Whang, Denton, Ghasemipour, Gontijo~Lopes, Karagol~Ayan, Salimans, et~al.]{saharia2022photorealistic}
Chitwan Saharia, William Chan, Saurabh Saxena, Lala Li, Jay Whang, Emily~L Denton, Kamyar Ghasemipour, Raphael Gontijo~Lopes, Burcu Karagol~Ayan, Tim Salimans, et~al.
\newblock Photorealistic text-to-image diffusion models with deep language understanding.
\newblock In \emph{Advances in Neural Information Processing Systems}, pages 36479--36494, 2022.

\bibitem[Salimans and Ho(2022)]{salimans2022progressive}
Tim Salimans and Jonathan Ho.
\newblock Progressive distillation for fast sampling of diffusion models.
\newblock In \emph{International Conference on Learning Representations}, 2022.

\bibitem[Salmona et~al.(2022)Salmona, De~Bortoli, Delon, and Desolneux]{salmona2022can}
Antoine Salmona, Valentin De~Bortoli, Julie Delon, and Agn{\`e}s Desolneux.
\newblock Can push-forward generative models fit multimodal distributions?
\newblock \emph{Advances in Neural Information Processing Systems}, 35:\penalty0 10766--10779, 2022.

\bibitem[S{\"a}rkk{\"a} and Solin(2019)]{sarkka2019applied}
Simo S{\"a}rkk{\"a} and Arno Solin.
\newblock \emph{Applied stochastic differential equations}.
\newblock Cambridge University Press, 2019.

\bibitem[Sohl-Dickstein et~al.(2015)Sohl-Dickstein, Weiss, Maheswaranathan, and Ganguli]{sohl2015deep}
Jascha Sohl-Dickstein, Eric Weiss, Niru Maheswaranathan, and Surya Ganguli.
\newblock Deep unsupervised learning using nonequilibrium thermodynamics.
\newblock In \emph{International conference on machine learning}, pages 2256--2265. PMLR, 2015.

\bibitem[Song et~al.(2021{\natexlab{a}})Song, Meng, and Ermon]{song2021ddim}
Jiaming Song, Chenlin Meng, and Stefano Ermon.
\newblock Denoising diffusion implicit models.
\newblock In \emph{International Conference on Learning Representations}, 2021{\natexlab{a}}.

\bibitem[Song and Ermon(2019)]{song2019ncsn}
Yang Song and Stefano Ermon.
\newblock Generative modeling by estimating gradients of the data distribution.
\newblock In \emph{Advances in Neural Information Processing Systems}, 2019.

\bibitem[Song et~al.(2021{\natexlab{b}})Song, Sohl{-}Dickstein, Kingma, Kumar, Ermon, and Poole]{song2021sde}
Yang Song, Jascha Sohl{-}Dickstein, Diederik~P. Kingma, Abhishek Kumar, Stefano Ermon, and Ben Poole.
\newblock Score-based generative modeling through stochastic differential equations.
\newblock In \emph{International Conference on Learning Representations}, 2021{\natexlab{b}}.

\bibitem[Song et~al.(2023)Song, Dhariwal, Chen, and Sutskever]{song2023consistency}
Yang Song, Prafulla Dhariwal, Mark Chen, and Ilya Sutskever.
\newblock Consistency models.
\newblock In \emph{International Conference on Machine learning}, 2023.

\bibitem[Vershynin(2018)]{vershynin2018high}
Roman Vershynin.
\newblock \emph{High-dimensional probability: An introduction with applications in data science}.
\newblock Cambridge university press, 2018.

\bibitem[Watson et~al.(2021)Watson, Chan, Ho, and Norouzi]{watson2021learning}
Daniel Watson, William Chan, Jonathan Ho, and Mohammad Norouzi.
\newblock Learning fast samplers for diffusion models by differentiating through sample quality.
\newblock In \emph{International Conference on Learning Representations}, 2021.

\bibitem[Xia et~al.(2023)Xia, Shen, Lei, Zhou, Yi, Zhao, Wang, and Liu]{xia2023towards}
Mengfei Xia, Yujun Shen, Changsong Lei, Yu Zhou, Ran Yi, Deli Zhao, Wenping Wang, and Yong-jin Liu.
\newblock Towards more accurate diffusion model acceleration with a timestep aligner.
\newblock \emph{arXiv preprint arXiv:2310.09469}, 2023.

\bibitem[Yu et~al.(2015)Yu, Seff, Zhang, Song, Funkhouser, and Xiao]{yu2015lsun}
Fisher Yu, Ari Seff, Yinda Zhang, Shuran Song, Thomas Funkhouser, and Jianxiong Xiao.
\newblock Lsun: Construction of a large-scale image dataset using deep learning with humans in the loop.
\newblock \emph{arXiv preprint arXiv:1506.03365}, 2015.

\bibitem[Zhang and Chen(2023)]{zhang2023deis}
Qinsheng Zhang and Yongxin Chen.
\newblock Fast sampling of diffusion models with exponential integrator.
\newblock In \emph{International Conference on Learning Representations}, 2023.

\bibitem[Zhao et~al.(2023)Zhao, Bai, Rao, Zhou, and Lu]{zhao2023unipc}
Wenliang Zhao, Lujia Bai, Yongming Rao, Jie Zhou, and Jiwen Lu.
\newblock Unipc: A unified predictor-corrector framework for fast sampling of diffusion models.
\newblock \emph{arXiv preprint arXiv:2302.04867}, 2023.

\end{thebibliography}
}
\appendix
\clearpage
\maketitlesupplementary


\section{Related Works}
\label{sec:sup_related_works}
Ever since the birth of diffusion models~\cite{ho2020ddpm,song2019ncsn}, their generation speed has become a major drawback compared to other generative models~\cite{goodfellow2014generative,kingma2013auto}.
To address this issue, efforts have been taken to accelerate the sampling of diffusion models, which fall into two main streams.

One is designing faster solvers. In early works~\cite{nichol2021improved,song2021ddim}, the authors speed up the generation from 1000 to less than 50 NFE by reducing the number of time steps with systematic or quadratic sampling. Analytic-DPM~\cite{bao2022analytic}, provides an analytic form of optimal variance in sampling process and improves the results. 
More recently, with the knowledge of interpreting the diffusion process as a PF-ODE~\cite{song2021sde}, there is a class of ODE solvers based on numerical methods that accelerate the sampling process to around 10 NFE. 
The authors in EDM~\cite{karras2022edm} achieve several improvements on training and sampling of diffusion models and propose to use Heun's second method. 
PDNM~\cite{liu2022pseudo} uses linear multi-step method to solve the PF-ODE with Runge-Kutta algorithms for the warming start. The authors in~\cite{zhang2023deis} recommend to use lower order linear multi-step method for warming start and propose iPNDM. 
Given the semi-linear structure of the PF-ODE, DPM-Solver~\cite{lu2022dpm} and DEIS~\cite{zhang2023deis} are proposed by approximate the integral involved in the analytic solution of PF-ODE with Taylor expansion and polynomial extrapolation respectively. 
DPM-Solver is further extend to both single-step and multi-step methods in DPM-Solver++~\cite{lu2022dpmpp}. 
UniPC~\cite{zhao2023unipc} gives a unified predictor-corrector solver and improved results compared to DPM-Solver++. 

Besides training-free fast solvers above, there are also solvers requiring additional training. ~\cite{watson2021learning} proposes a series of reparameterization for a generalized family of DDPM with a KID~\cite{binkowski2018demystifying} loss. More related to our method, in GENIE~\cite{dockhorn2022genie}, the authors apply the second truncated Taylor method~\cite{platen2010numerical} to the PF-ODE and distill a new model to predict the higher-order gradient term. Different from this, in AMED-Solver, we train a network that only predict the intermediate time steps, instead of a high-dimensional output.

Another mainstream is training-based distillation methods, which attempt to build a direct mapping from noise distribution to implicit data distribution. This idea is first introduced in~\cite{luhman2021knowledge} as an offline method where one needs to pre-construct a dataset generated by the original model.
Rectified flow~\cite{liu2022flow} also introduces an offline distillation based on optimal transport. For online distillation methods, one can progressively distill a diffusion model from more than 1k steps to 1 step~~\cite{salimans2022progressive,berthelot2023tract}, or utilize the consistency property of PF-ODE trajectory to tune the denoising output~\cite{daras2023consistent,song2023consistency,gu2023boot}.

\section{Experimental Details}
\label{sec:sup_exp_details}

\noindent
\textbf{Datasets.} We employ AMED-Solver and AMED-Plugin on a wide range of datasets and settings. We report results on settings including unconditional generation in both pixel and latent space, conditional generation with or without guidance. Datasets are chosen with image resolutions ranging from 32 to 512, including CIFAR10 32$\times$32~\cite{krizhevsky2009learning}, FFHQ 64$\times$64~\cite{karras2019style}, ImageNet 64$\times$64~\cite{russakovsky2015ImageNet} and LSUN Bedroom 256$\times$256~\cite{yu2015lsun}. We give quantitative and qualitative results generated by stable-diffusion~\cite{rombach2022ldm} with resolution of 512.
To further evaluate the effectiveness of our methods, in \cref{sec:sup_additional_results}, we also include more results on ImageNet 256$\times$256 ~\cite{russakovsky2015ImageNet} with classifier guidance and latent-space LSUN Bedroom 256$\times$256~\cite{yu2015lsun}.

\noindent
\textbf{Models.} The pre-trained models we use throughout our experiments are pixel-space models from~\cite{karras2022edm},~\cite{song2023consistency} as well as~\cite{dhariwal2021diffusion}, and latent-space models from~\cite{rombach2022ldm}. Our code architecture is mainly based on the implementation in~\cite{karras2022edm}.

\noindent
\textbf{Fast ODE solvers.} To give fair comparison, we reimplement several representative fast ODE solvers including DDIM~\cite{song2021ddim}, DPM-Solver-2~\cite{lu2022dpm}, multi-step DPM-Solver++~\cite{lu2022dpmpp}, UniPC~\cite{zhao2023unipc} and improved PNDM (iPNDM)~\cite{liu2022pseudo,zhang2023deis}. Through the implementation, we obtain better or on par FID results compared with original papers. It is worth mentioning that during the implementation, we find iPNDM achieves very impressive results and outperforms other ODE solvers in many cases.

\noindent
\textbf{Time schedule.} We notice that different ODE solvers have different preference on time schedule. We mainly use the polynomial time schedule with $\rho=7$, which is the default setting in~\cite{karras2022edm}, except for DPM-Solver++ and UniPC where we use logSNR time schedule recommended in original papers~\cite{lu2022dpmpp,zhao2023unipc} for better results. 
As illustrated in \cref{subsec:training_and_sampling}, the logSNR schedule is the limit case of polynomial schedule as $\rho$ approaches $+\infty$. 
Besides, for AMED-Solver on CIFAR10 32$\times$32~\cite{krizhevsky2009learning}, FFHQ 64$\times$64~\cite{karras2019style} and ImageNet 64$\times$64~\cite{russakovsky2015ImageNet}, we use uniform time schedule which is widely used in papers with a DDPM~\cite{ho2020ddpm} backbone. This uniform time schedule is transferred from its original range $[0.001, 1]$ to $[t_1, t_N]$ in our setting following the EDM~\cite{karras2022edm} implementation.
For experiments on ImageNet 256$\times$256 with classifier guidance and latent-space LSUN Bedroom, the uniform time schedule gives best results. The reason may lie in their different training process.

\noindent
\textbf{Training.} Since there are merely 9k parameters in the AMED predictor $g_\phi$, its training does not cause much computational cost. The training process spends its main time on generating student and teacher trajectories. We train $g_\phi$ for 10k images, which takes 2-8 minutes on CIFAR10 and 1-3 hours on LSUN Bedroom using a single NVIDIA A100 GPU. For the distance metric in \cref{eq:loss}, we use L2 norm in all experiments.
For the generation of teacher trajectories for AMED-Solver, we use DPM-Solver-2~\cite{lu2022dpm} or EDM~\cite{karras2022edm} with doubled NFE ($M = 1$). For AMED-Plugin, we use the same solver that generates student trajectories with $M = 1$ for DPM-Solver-2 and $M = 2$ else.

\noindent
\textbf{Sampling.} Due to designation, our AMED-Solver or AMED-Plugin naturally create solvers with even NFE. Therefore once AFS is used, the total NFE become odd. With the goal of designing fast ODE solvers in extremely small NFE, we mainly test our method on NFE $\in \{3, 5, 7, 9\}$ where AFS is applied. There are also results on NFE $\in \{4, 6, 8, 10\}$ without using AFS.

\noindent
\textbf{Evaluation.} We measure the sample quality via Frechet Inception Distance (FID)~\cite{heusel2017gans}, which is a well-known quantitative evaluation metric for image quality that aligns well with human perception. For all the experiments involved, we calculate FID with 50k samples using the implementation in~\cite{karras2022edm}. For Stable-Diffusion, we follow~\cite{ramesh2021zero} and evaluate the FID value by 30k samples generated by 30k fixed prompts sampled from the MS-COCO~\cite{lin2014microsoft} validation set.

\section{Additional Results}
\label{sec:sup_additional_results}

\subsection{Fast Degradation of Single-step Solvers}
\label{subsec:sup_degradation}
In pilot experiments, we find that the fast degradation of single-step solvers can be alleviated by appropriate choose of the intermediate time steps. As shown in \cref{tab:fid_degeadation}, the performance of DPM-Solver-2 is very sensitive to the choice of its hyperparameter $r$. We apply AMED-Plugin to learn the appropriate $r$ and find that it achieves similar results with the best searched $r$ while adding little training overhead and negligible sampling overhead. 

\renewcommand{\arraystretch}{0.8}
\begin{table}[htbp]
    \centering
    \fontsize{8}{10}\selectfont
    \begin{tabular}{lcccc}
      \toprule
      \multirow{2}{*}{$r$} & \multicolumn{4}{c}{NFE} \\
      \cmidrule{2-5}
      & 4 & 6 & 8 & 10 \\
      \midrule
      0.1           & 28.75 & 17.61 & 11.00 & 5.30 \\
      0.2           & 24.44 & 21.94 & 8.23  & 4.17 \\
      0.3           & 37.31 & 30.27 & 6.82  & \textbf{3.89} \\
      0.4           & 75.06 & 43.32 & 7.34  & 4.20 \\
      0.5 (default) & 146.0 & 60.00 & 10.30 & 5.01 \\
      0.6           & 241.2 & 79.27 & 15.62 & 6.38 \\
      \textbf{AMED-Plugin}  & \textbf{24.44} & \textbf{17.10} & \textbf{6.60} & 4.73 \\
      \bottomrule
    \end{tabular}
    \caption{Performance of DPM-Solver-2 on CIFAR10 is sensitive to the choice of $r$. Applying AMED-Plugin on DPM-Solver-2 efficiently help to learn the appropriate $r$.}
    \label{tab:fid_degeadation}
\end{table}

\subsection{Ablation Study on Intermediate Steps}
\label{subsec:sup_ablation_steps}
As illustrated in \cref{subsec:training_and_sampling}, for the teacher sampling trajectory, $M$ intermediate time steps are injected in every sampling step. 
We use a smooth interpolation, meaning that the teacher time schedule given by the original schedule $\Gamma$ combined with injected time steps, is equivalent to the time schedule obtained by simply setting the total number of time steps to $(M+1)(N-1)+1$ in \cref{eq:poly_schedule} under the same $\rho$. In this way, we can easily extract samples on teacher trajectories at $\Gamma$ to get reference samples $\{\bfy_{t_n}\}_{n=1}^{N}$.

Here we take unconditional generation on CIFAR10 using iPNDM with AMED-Plugin as an example and provide an ablation study on the choose of $M$. The student and teacher solvers are set to be the same. The results are shown in \cref{tab:sup_ablation_m}.

\renewcommand{\arraystretch}{0.8}
\begin{table}[htbp]
  \fontsize{8}{10}\selectfont
    \centering
    \begin{tabular}{lccccc}
        \toprule
        \multirow{2}{*}{$M$} & \multicolumn{4}{c}{NFE} \\
        \cmidrule{2-5}
        & 3 & 5 & 7 & 9 \\
        \midrule
        1 & 25.98 & 7.53 & 4.01 & 2.77 \\
        2 & \textbf{10.81} & \textbf{6.61} & \textbf{3.65} & 2.63 \\
        3 & 11.20 & 7.00 & 4.15 & 2.63 \\
        4 & 10.92 & 7.40 & 3.69 & \textbf{2.61} \\
        5 & 12.87 & 7.70 & 3.65 & 2.75 \\
        \bottomrule
    \end{tabular}
  \caption{The sensitivity of $M$ on CIFAR10 with AMED-Plugin on iPNDM. Additional time scaling factors are trained.}
  \label{tab:sup_ablation_m}
\end{table}

\subsection{Ablation Study on Bottleneck Feature Input and Time Scaling Factor}
\label{subsec:sup_ablation_bottleneck}
As the U-Net bottleneck input to $g_\phi$ varies for different samples, the learned parameters are sample-wise, meaning that different trajectories have different time schedules. 
Since sampling trajectories from different starting points share similar geometric shapes, the effectiveness of inputting bottleneck might be limited. We also notice that during training, the standard deviation of learned parameters in one batch is small. 
Therefore, we should test if the U-Net bottleneck feature input is necessary.
In \cref{tab:sup_ablation_bottleneck} we replace U-Net bottleneck feature with zero matrix (w/o bottleneck) and get shared parameters across all sampling trajectories. The results show the effectiveness of the bottleneck feature input.
Besides, the use of time scaling factors provides improved results. When applying our AMED-Plugin on DDIM~\cite{song2021ddim}, iPNDM~\cite{zhang2023deis} and DPM-Solver++~\cite{lu2022dpmpp} on datasets with small resolution (32$\times$32 and 64$\times$64), we optionally train this time scaling factors through $g_\phi$ to expand the solution space.

\renewcommand{\arraystretch}{0.8}
\begin{table}[htbp]
    \centering
    \fontsize{8}{10}\selectfont
    \begin{tabular}{lcccc}
      \toprule
      \multirow{2}{*}{Method} & \multicolumn{4}{c}{NFE} \\
      \cmidrule{2-5}
      & 3 & 5 & 7 & 9 \\
      \midrule
      iPNDM~\cite{liu2022pseudo,zhang2023deis}    & 47.98 & 13.59 & 5.08 & 3.17 \\
      AMED-Plugin (w/ bottleneck)                 & 15.87 & 7.29  & 3.92 & 2.82 \\
      AMED-Plugin (w/o bottleneck) $\dagger$      & 11.12 & 7.31  & 3.80 & 2.64 \\
      AMED-Plugin (w/ bottleneck)  $\dagger$      & \textbf{10.81} & \textbf{6.61} & \textbf{3.65} & \textbf{2.63} \\
      \bottomrule
    \end{tabular}
    \caption{Ablation study of the bottleneck feature input on CIFAR10 with AMED-Plugin on iPNDM. $\dagger$: additional time scaling factors are trained.}
    \label{tab:sup_ablation_bottleneck}
\end{table}

\subsection{Ablation Study on AFS}
\label{subsec:sup_ablation_afs}
The trick of analytical first step (AFS) is first introduced in~\cite{dockhorn2022genie} to reduce one NFE, where the authors replace the U-Net output in the first sampling step with the direction of $x_T$. In \cref{tab:sup_fid_cifar10} and \cref{tab:sup_fid_imagenet64}, we provide extended results of \cref{subtab:fid1} and \cref{subtab:fid2} as well as the ablations between AFS and our proposed AMED-Plugin. We find that the use of AFS provides consistent improvement on datasets with resolutions of 32 and 64. The results show that in most cases they can be considered as two independent components that can together boost the performance of various ODE solvers. However, for datasets with large resolutions, applying AFS usually causes a large degradation (see \cref{tab:sup_ablation_bedroom}).

\renewcommand{\arraystretch}{0.8}
\begin{table}[htbp]
  \fontsize{8}{10}\selectfont
    \centering
    \begin{tabular}{lcccccc}
        \toprule
        \multirow{2}{*}{Method} & \multirow{2}{*}{AFS} & \multicolumn{4}{c}{NFE} \\
        \cmidrule{3-6}
        & & 3 & 5 & 7 & 9 \\
        \midrule
        \multirow{2}{*}{DPM-Solver++(3M)~\cite{lu2022dpmpp}}
        & \usym{2613}  & 111.9 & 23.15 & 8.87  & 6.45 \\
        & \usym{1F5F8} & 127.5 & 25.04 & 10.51 & 7.32 \\
        \midrule
        \multirow{2}{*}{UniPC~\cite{zhao2023unipc}}
        & \usym{2613}  & 112.3 & 23.34 & 8.73  & 6.61 \\
        & \usym{1F5F8} & 127.3 & 25.78 & 10.50 & 7.20 \\
        \midrule
        \multirow{2}{*}{iPNDM~\cite{liu2022pseudo,zhang2023deis}}
        & \usym{2613}  & 80.99 & 26.65 & 13.80 & 8.38 \\
        & \usym{1F5F8} & 95.61 & 34.61 & 21.96 & 10.06 \\
        \midrule
        \multirow{2}{*}{DPM-Solver-2~\cite{lu2022dpm}}
        & \usym{2613}  & 210.6 & 80.60 & 23.25 & 9.61 \\
        & \usym{1F5F8} & 241.8 & 88.79 & 22.59 & 9.07 \\
        \midrule
        AMED-Solver (\textbf{ours})
        & \usym{1F5F8} & \textbf{58.51} & \textbf{13.20} & \textbf{7.10} & \textbf{5.65} \\
        \bottomrule
    \end{tabular}
  \caption{Ablation study of AFS on pixel-space LSUN Bedroom 256$\times$256.}
  \label{tab:sup_ablation_bedroom}
\end{table}

\renewcommand{\arraystretch}{0.85}
\begin{table*}[htbp]
  \fontsize{8}{11}\selectfont
    \centering
    \begin{tabular}{lcccccccccc}
        \toprule
        \multirow{2}{*}{Method} & \multirow{2}{*}{AFS} & \multirow{2}{*}{\textbf{AMED}} & \multicolumn{8}{c}{NFE} \\
        \cmidrule{4-11}
        & & & 3 & 4 & 5 & 6 & 7 & 8 & 9 & 10 \\
        \midrule
        \multicolumn{4}{l}{\fontsize{7}{1}\selectfont \textbf{Multi-step solvers}} \\
        \midrule
        \multirow{2}{*}{UniPC~\cite{zhao2023unipc}}
        & \usym{2613}  & \usym{2613}  & 109.6 & 45.20 & 23.98 & 11.14 & 5.83 & 3.99 & 3.21 & 2.89  \\
        & \usym{1F5F8} & \usym{2613}  & 54.36 & 20.55 & 9.01  & 5.75  & 4.11 & 3.26 & 2.93 & 2.65 \\
        \midrule
        \multirow{4}{*}{DPM-Solver++(3M)~\cite{lu2022dpmpp}$\dagger$}
        & \usym{2613}  & \usym{2613}  & 110.0 & 46.52 & 24.97 & 11.99 & 6.74 & 4.54 & 3.42 & 3.00 \\
        & \usym{1F5F8} & \usym{2613}  & 55.74 & 22.40 & 9.94  & 5.97  & 4.29 & 3.37 & 2.99 & 2.71 \\
        & \usym{2613}  & \usym{1F5F8} & -     & 21.62 & -     & 6.82  & -    & 4.41 & -    & 2.76 \\
        & \usym{1F5F8} & \usym{1F5F8} & 25.95 & -     & 7.68  & -     & 4.51 & -    & 3.03 & -    \\
        \midrule
        \multirow{4}{*}{iPNDM~\cite{liu2022pseudo,zhang2023deis}$\dagger$}
        & \usym{2613}  & \usym{2613}  & 47.98 & 24.82 & 13.59 & 7.05 & 5.08 & 3.69 & 3.17 & 2.77 \\
        & \usym{1F5F8} & \usym{2613}  & 24.54 & 13.92 & 7.76  & \textbf{5.07} & 4.04 & \textbf{3.22} & 2.83 & 2.56 \\
        & \usym{2613}  & \usym{1F5F8} & - & \textbf{10.43} & - & 6.67 & - & 3.34 & - & \textbf{2.48} \\
        & \usym{1F5F8} & \usym{1F5F8} & \textbf{10.81} & - & \textbf{6.61} & - & \textbf{3.65} & - & \textbf{2.63} & - \\
        \midrule
        \fontsize{7}{1}\selectfont \textbf{Single-step solvers} \\
        \midrule
        \multirow{4}{*}{DDIM~\cite{song2021ddim}$\dagger$}
        & \usym{2613}  & \usym{2613}  & 93.36 & 66.76 & 49.66 & 35.62 & 27.93 & 22.32 & 18.43 & 15.69 \\
        & \usym{1F5F8} & \usym{2613}  & 67.26 & 49.96 & 35.78 & 28.00 & 22.37 & 18.48 & 15.69 & 13.47 \\
        & \usym{2613}  & \usym{1F5F8} & -     & 37.72 & -     & 25.15 & -     & 17.03 & -     & 11.33 \\
        & \usym{1F5F8} & \usym{1F5F8} & 38.23 & -     & 24.44 & -     & 15.72 & -     & 10.93 & -     \\
        \midrule
        \multirow{4}{*}{DPM-Solver-2~\cite{lu2022dpm}}
        & \usym{2613}  & \usym{2613}  & -     & 146.0 & -     & 60.00 & -     & 10.30 & -     & 5.01  \\
        & \usym{1F5F8} & \usym{2613}  & 155.7 & -     & 57.28 & -     & 10.20 & -     & 4.98  & -     \\
        & \usym{2613}  & \usym{1F5F8} & -     & 24.44 & -     & 17.10 & -     & 6.60  & -     & 4.73  \\
        & \usym{1F5F8} & \usym{1F5F8} & 38.48 & -     & 28.14 & -     & 7.46  & -     & 4.73  & -     \\
        \midrule
        \multirow{2}{*}{AMED-Solver (\textbf{ours})}
        & \usym{2613}  & \usym{1F5F8} & -     & \textbf{17.18} & -     & \textbf{7.04}  & -    & \textbf{5.56} &  -   & \textbf{4.14} \\
        & \usym{1F5F8} & \usym{1F5F8} & \textbf{18.49} & -     & \textbf{7.59} &  -    & \textbf{4.36} & -    & \textbf{3.67} &  -   \\
        \bottomrule
    \end{tabular}
  \caption{Unconditional generation on CIFAR10 32$\times$32. $\dagger$: additional time scaling factors $\{a_n\}_{n=1}^{N-1}$ are trained.}
  \label{tab:sup_fid_cifar10}
\end{table*}

\renewcommand{\arraystretch}{0.85}
\begin{table*}[htbp]
  \fontsize{8}{11}\selectfont
    \centering
    \begin{tabular}{lcccccccccc}
        \toprule
        \multirow{2}{*}{Method} & \multirow{2}{*}{AFS} & \multirow{2}{*}{\textbf{AMED}} & \multicolumn{8}{c}{NFE} \\
        \cmidrule{4-11}
        & & & 3 & 4 & 5 & 6 & 7 & 8 & 9 & 10 \\
        \midrule
        \multicolumn{4}{l}{\fontsize{7}{1}\selectfont \textbf{Multi-step solvers}} \\
        \midrule
        \multirow{2}{*}{UniPC~\cite{zhao2023unipc}}
        & \usym{2613}  & \usym{2613}  & 91.38 & 55.63 & 24.36 & 14.30 & 9.57  & 7.52 & 6.34 & 5.53 \\
        & \usym{1F5F8} & \usym{2613}  & 64.54 & 29.59 & 16.17 & 11.03 & 8.51  & 6.98 & 6.04 & 5.26 \\
        \midrule
        \multirow{4}{*}{DPM-Solver++(3M)~\cite{lu2022dpmpp}$\dagger$}
        & \usym{2613}  & \usym{2613}  & 91.52 & 56.34 & 25.49 & 15.06 & 10.14 & 7.84 & 6.48 & 5.67 \\
        & \usym{1F5F8} & \usym{2613}  & 65.20 & 30.56 & 16.87 & 11.38 & 8.68  & 7.12 & 6.25 & 5.58 \\
        & \usym{2613}  & \usym{1F5F8} & -     & 53.28 & -     & 13.68 & -     & 7.98 & -    & 5.57 \\
        & \usym{1F5F8} & \usym{1F5F8} & 76.51 & -     & 15.21 & -     & 8.36  & -    & 6.04 & -    \\
        \midrule
        \multirow{4}{*}{iPNDM~\cite{liu2022pseudo,zhang2023deis}}
        & \usym{2613}  & \usym{2613}  & 58.53 & 33.79 & 18.99 & 12.92 & 9.17 & 7.20 & 5.91 & 5.11 \\
        & \usym{1F5F8} & \usym{2613}  & 34.81 & \textbf{21.32} & 15.53  & \textbf{10.27} & 8.64 & \textbf{6.60} & 5.64 & \textbf{4.97} \\
        & \usym{2613}  & \usym{1F5F8} & - & 23.55 & - & 12.05 & - & 7.03 & - & 5.01 \\
        & \usym{1F5F8} & \usym{1F5F8} & \textbf{28.06} & - & \textbf{13.83} & - & \textbf{7.81} & - & \textbf{5.60} & - \\
        \midrule
        \fontsize{7}{1}\selectfont \textbf{Single-step solvers} \\
        \midrule
        \multirow{4}{*}{DDIM~\cite{song2021ddim}$\dagger$}
        & \usym{2613}  & \usym{2613}  & 82.96 & 58.43 & 43.81 & 34.03 & 27.46 & 22.59 & 19.27 & 16.72 \\
        & \usym{1F5F8} & \usym{2613}  & 62.42 & 46.06 & 35.48 & 28.50 & 23.31 & 19.82 & 17.14 & 15.02 \\
        & \usym{2613}  & \usym{1F5F8} & -     & 40.85 & -     & 32.46 & -     & 20.72 & -     & 15.52 \\
        & \usym{1F5F8} & \usym{1F5F8} & 46.10 & -     & 33.54 & -     & 21.94 & -     & 15.56 & -     \\
        \midrule
        \multirow{4}{*}{DPM-Solver-2~\cite{lu2022dpm}}
        & \usym{2613}  & \usym{2613}  & -     & 129.8 & -     & 44.83 & -     & 12.42 & -     & 6.84  \\
        & \usym{1F5F8} & \usym{2613}  & 140.2 & -     & 42.41 & -     & 12.03 & -     & 6.64  & -     \\
        & \usym{2613}  & \usym{1F5F8} & -     & 40.99 & -     & 31.19 & -     & 11.24 & -     & 6.94  \\
        & \usym{1F5F8} & \usym{1F5F8} & 70.64 & -     & 29.96 & -     & 11.54 & -     & 6.91  & -     \\
        \midrule
        \multirow{2}{*}{AMED-Solver (\textbf{ours})}
        & \usym{2613}  & \usym{1F5F8} & -     & \textbf{32.69} & -     & \textbf{10.63}  & -    & \textbf{7.71} &  -   & \textbf{6.06} \\
        & \usym{1F5F8} & \usym{1F5F8} & \textbf{38.10} & -     & \textbf{10.74} &  -    & \textbf{6.66} & -    & \textbf{5.44} &  -   \\
        \bottomrule
    \end{tabular}
  \caption{Conditional generation on ImageNet 64$\times$64. $\dagger$: additional time scaling factors $\{a_n\}_{n=1}^{N-1}$ are trained.}
  \label{tab:sup_fid_imagenet64}
\end{table*}

\subsection{More Quantitative Results}
\label{subsec:sup_quantitative}
In this section, we provide additional quantitative results on more datasets including latent-space LSUN Bedroom 256$\times$256~\cite{yu2015lsun,rombach2022ldm} and ImageNet 256$\times$256~\cite{russakovsky2015ImageNet} with classifier guidance~\cite{dhariwal2021diffusion}. The results are shown in \cref{tab:sup_fid_latent_bedroom} and \cref{tab:sup_fid_guidance}.

\renewcommand{\arraystretch}{0.8}
\begin{table}[htbp]
  \centering
  \fontsize{8}{10}\selectfont
    \centering
    \begin{tabular}{lcccc}
      \toprule
      \multirow{2}{*}{Method} & \multicolumn{4}{c}{NFE} \\
      \cmidrule{2-5}
      & 4 & 6 & 8 & 10 \\
      \midrule
      DPM-Solver++(3M)~\cite{lu2022dpmpp}       & 48.55 & 10.01 & 4.61 & 3.62 \\
      AMED-Plugin (\textbf{ours})               & \textbf{15.67} & \textbf{8.92} & \textbf{4.19} & \textbf{3.52} \\
      \bottomrule
    \end{tabular}
  \caption{Unconditional generation on latent-space LSUN Bedroom. AMED-Plugin is applied on DPM-Solver++.}
  \label{tab:sup_fid_latent_bedroom}
\end{table}

\renewcommand{\arraystretch}{0.8}
\begin{table}[htbp]
  \centering
  \fontsize{8}{10}\selectfont
    \centering
    \begin{tabular}{lcccc}
      \toprule
      \multirow{2}{*}{Method} & \multicolumn{4}{c}{NFE} \\
      \cmidrule{2-5}
      & 4 & 6 & 8 & 10 \\
      \midrule
      \fontsize{7}{1}\selectfont \textbf{Guidance scale = 8.0} \\
      \midrule
      DPM-Solver++(3M)~\cite{lu2022dpmpp}       & 60.01 & 25.51 & \textbf{11.98} & \textbf{7.95} \\
      AMED-Plugin (\textbf{ours})               & \textbf{39.84} & \textbf{21.79} & 13.94 & 9.05 \\
      \midrule
      \fontsize{7}{1}\selectfont \textbf{Guidance scale = 4.0} \\
      \midrule
      DPM-Solver++(3M)~\cite{lu2022dpmpp}       & 27.15 & 10.25 & 7.10 & 6.15 \\
      AMED-Plugin (\textbf{ours})               & \textbf{24.19} & \textbf{8.86} & \textbf{6.54} & \textbf{5.72} \\
      \midrule
      \fontsize{7}{1}\selectfont \textbf{Guidance scale = 2.0} \\
      \midrule
      DPM-Solver++(3M)~\cite{lu2022dpmpp}       & \textbf{23.06} & 10.17 & 7.04 & 5.92 \\
      AMED-Plugin (\textbf{ours})               & 28.81 & \textbf{9.56} & \textbf{6.42} & \textbf{5.46} \\
      \bottomrule
    \end{tabular}
  \caption{Conditional generation on ImageNet256 with classifier guidance. AMED-Plugin is applied on DPM-Solver++.}
  \label{tab:sup_fid_guidance}
\end{table}

\subsection{More Qualitative Results}
\label{subsec:sup_qualitative}
We give more qualitative results generated by stable-diffusion-v1~\cite{rombach2022ldm} with a default classifier-free guidance scale 7.5 in \cref{fig:sup_sd}. Results on various datasets with NFE of 3 and 5 are provided from \cref{fig:sup_grid_cifar10_3} to \cref{fig:sup_grid_ffhq64_5}.

\begin{figure}
  \centering
  \begin{subfigure}[b]{\linewidth}
      \includegraphics[width=\linewidth]{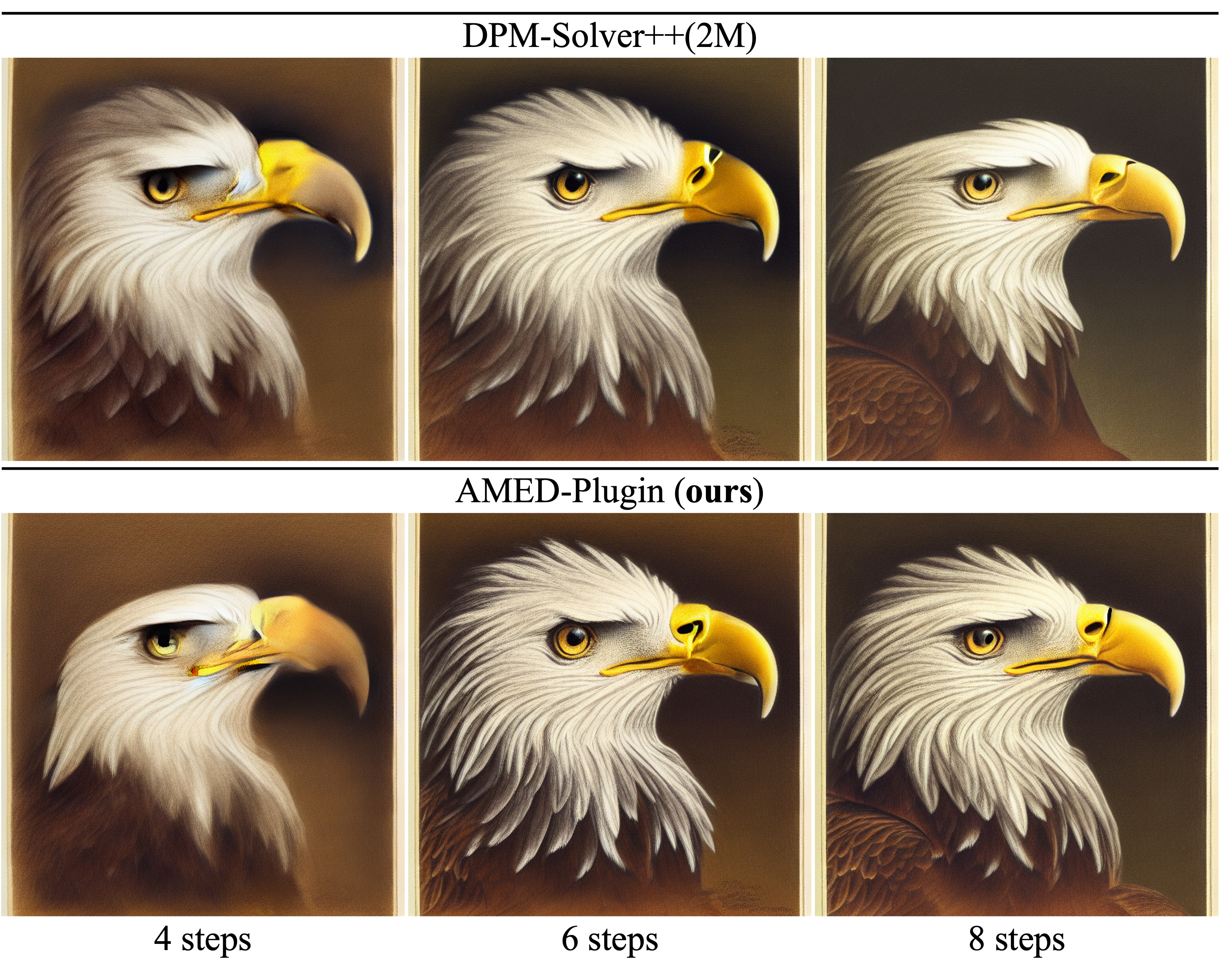}
      \caption{Text prompt: "\textit{A portrait of an eagle}".}
  \end{subfigure}
  \begin{subfigure}[b]{\linewidth}
      \includegraphics[width=\linewidth]{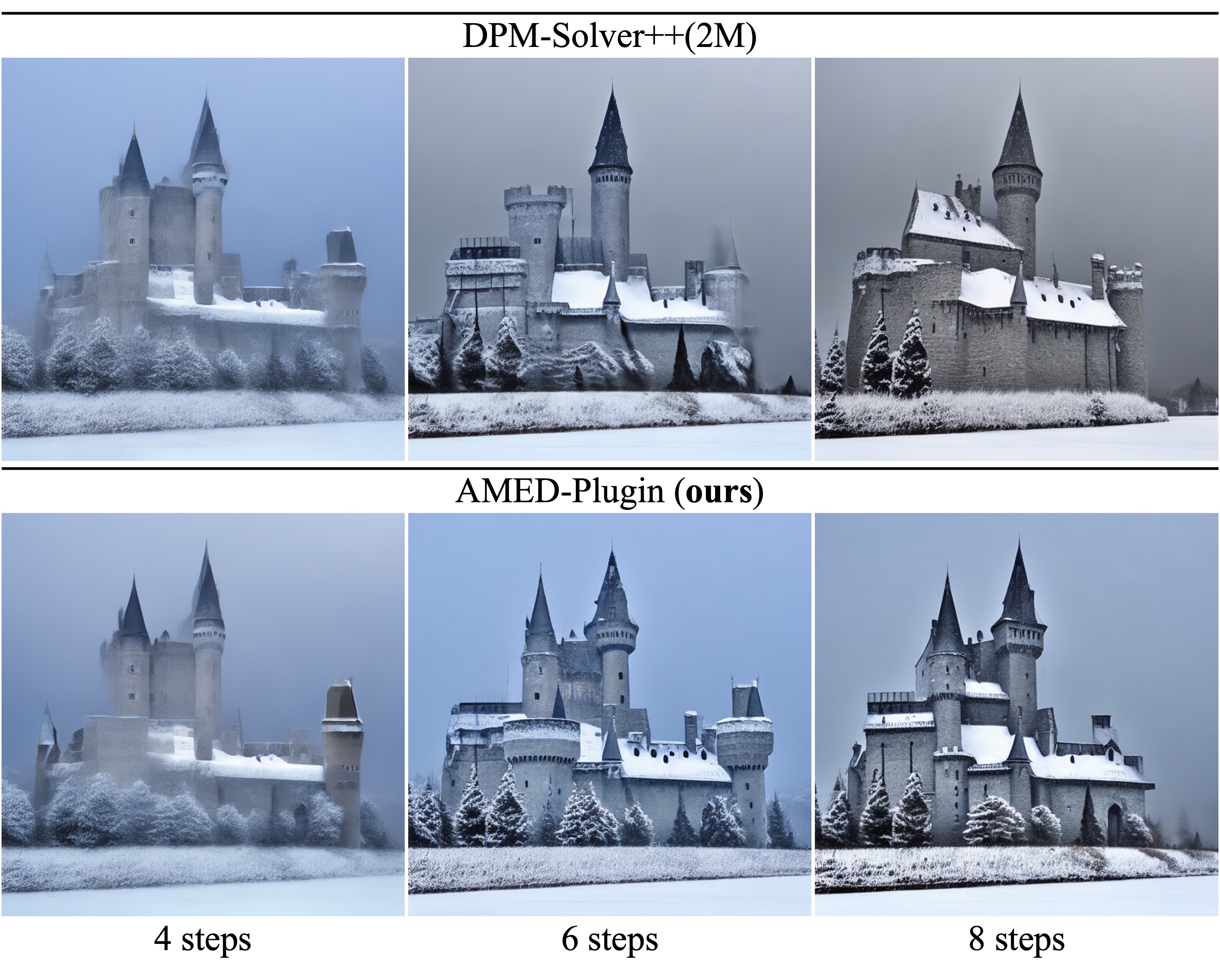}
      \caption{Text prompt: "\textit{Under a gray sky, a castle in ice and snow}".}
  \end{subfigure}
  \begin{subfigure}[b]{\linewidth}
      \includegraphics[width=\linewidth]{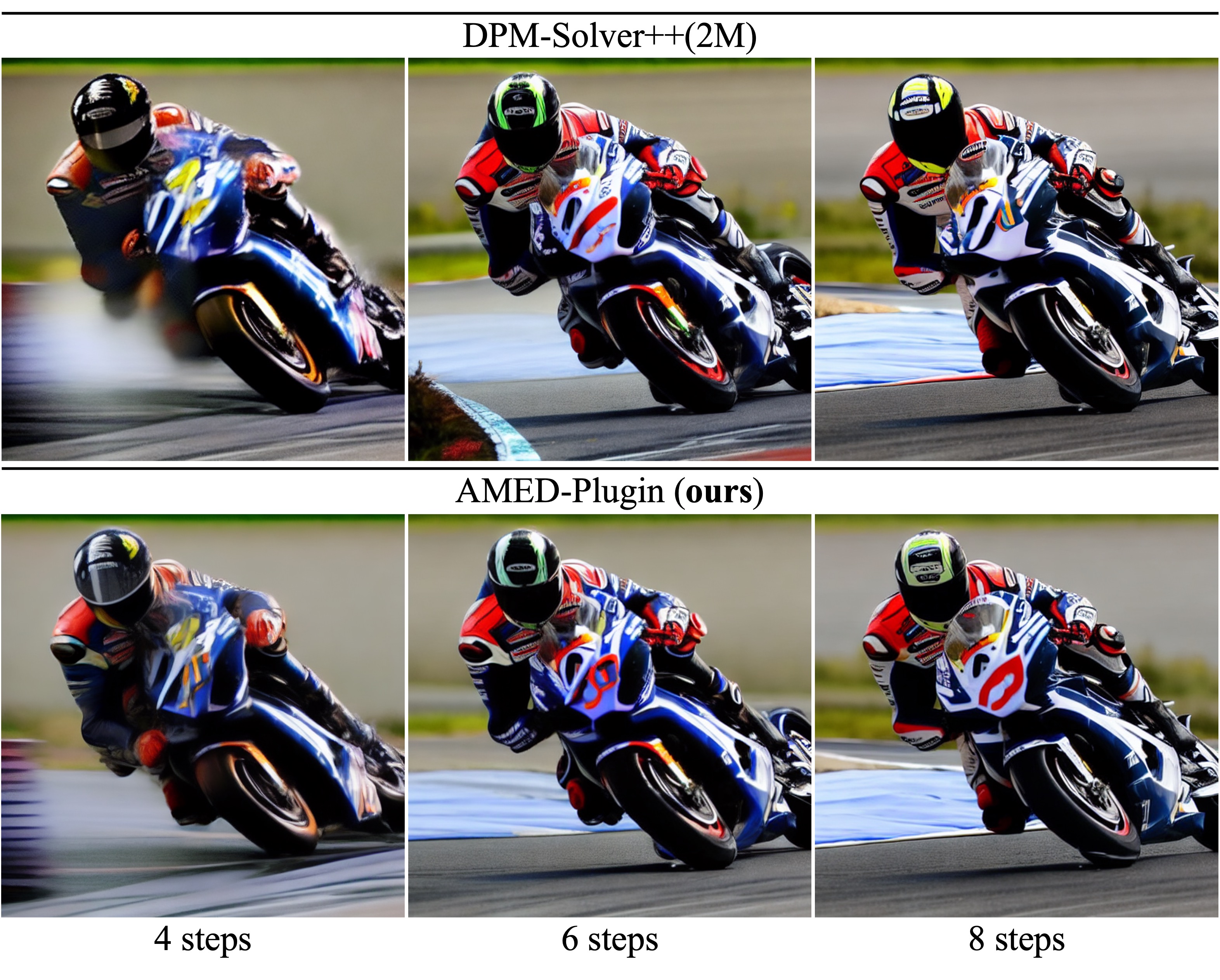}
      \caption{Text prompt: "\textit{A motorcycle racer is bending}".}
  \end{subfigure}
  \caption{Synthesized images by Stable-Diffusion v1.4~\cite{rombach2022ldm} with default classifier-free guidance scale 7.5.}
  \label{fig:sup_sd}
\end{figure}

\section{Theoretical Analysis}
\label{sec:sup_theoretical}
In \cref{subsec:geometric}, we experimentally showed that the sampling trajectory of diffusion models generated by an ODE solver almost lies in a two-dimensional subspace embedded in the ambient space. This is the core condition for the mean value theorem to approximately hold in the vector-valued function case. However, the sampling trajectory would not necessarily lie in a plane. In this section, we analyze to what extent will this affects our AMED-Solver, where we set the two-dimensional subspace to be the place spanned by the first two principal components.

\begin{figure*}
    \centering
    \includegraphics[width=\textwidth]{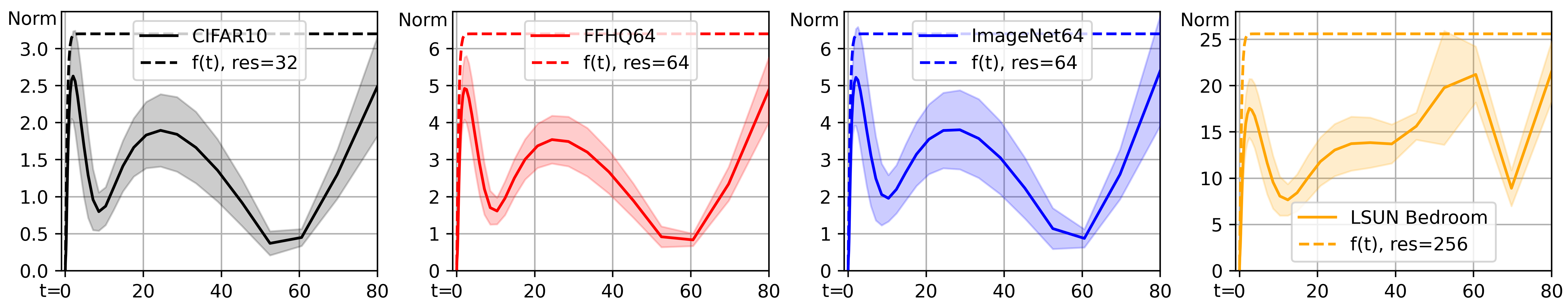}
    \caption{Following the experiment in \cref{subsec:geometric}, we calculate $\left\|\bfx_{t}-\tilde{\bfx}_{t}\right\|_2$ and find that we can bound it by a proper setting of \cref{eq:sup_logsitic}.}
    \label{fig:sup_pca}
\end{figure*}

\noindent
\textbf{Notations.} Denote $d$ as the dimension of the ambient space. Let $\{\bfx_{\tau}\}_{\tau=\epsilon}^{T}$ to be the solution of the PF-ODE \cref{eq:plug-in ODE}. Let $\{\tilde{\bfx}_{\tau}\}_{\tau=\epsilon}^{T}$ to be the trajectory obtained by projecting $\{\bfx_{\tau}\}_{\tau=\epsilon}^{T}$ to the two-dimensional subspace spanned by its first two principal components. 
Given $\epsilon \le s < m < t \le T$ and a constant $c$, one step of the AMED-Solver is given by
\begin{equation}
    \label{eq:sup_amed}
    \bfx_{s}^A = \bfx_{t} + c (s - t) \bfeps_{\theta}(\bfx_{m}, m).
\end{equation}
Define the scaled logistic function to be 
\begin{equation}
    \label{eq:sup_logsitic}
    f(\tau) = a \left(\frac{1}{1+e^{-b\tau}} - \frac{1}{2}\right), \tau \in \bbR; a,b \in \bbR^+.
\end{equation}
Finally, define a SDE 
\begin{equation}
    \label{eq:sde_assump4}
    \rmd \bfz_\tau = g(\tau) \rmd \bfw_\tau, \tau \in [s, t]
\end{equation}
with initial value $\bf0$ at $t$ where $g(\tau)$ is a real-valued function and $\bfw_\tau \in \bbR^d$ is the standard Wiener process.

We start by the following assumptions:
\begin{assumption}
    \label{assump:assump1}
    Assume that there exists $a, b > 0$ s.t. $\left\|\bfx_{\tau}-\tilde{\bfx}_{\tau}\right\|_2 \le f(\tau), \tau \in [\epsilon, T]$.
\end{assumption}

For the choice of $a$ and $b$, in \cref{fig:sup_pca}, we calculate $\left\|\bfx_{\tau}-\tilde{\bfx}_{\tau}\right\|_2$ following the experiment settings in \cref{subsec:geometric}. We experimentally find that $\left\|\bfx_{\tau}-\tilde{\bfx}_{\tau}\right\|_2$ can be roughly upper bounded by setting $a=\sqrt{3d}/15$ and $b=3$.

\begin{assumption}
    \label{assump:assump2}
    Assume that there exists an integrable function $\bfomega: \bbR^{d+1} \to \bbR^d$ that generate $\{\tilde{\bfx}_{\tau}\}_{\tau=\epsilon}^{T}$ by
    \begin{equation}
        \label{eq:sup_pca_solution}
        \tilde{\bfx}_{s} = \tilde{\bfx}_{t} + \int_{t}^{s} \bfeps_{\theta}(\bfx_\tau, \tau) + \bfomega(\bfx_\tau, \tau)\rmd \tau.
    \end{equation}
    with initial value $\tilde{\bfx}_{T}$. In this way, we can decompose the integral in \cref{eq:solution} into two components that parallel and perpendicular to the plane where $\{\tilde{\bfx}_{\tau}\}_{\tau=\epsilon}^{T}$ lies, i.e.,
    \begin{equation}
    \begin{aligned}
        \int_{t}^{s} \bfeps_{\theta}(\bfx_\tau, \tau)\rmd \tau =
        & \int_{t}^{s} \bfeps_{\theta}(\bfx_\tau, \tau) \bfomega(\bfx_\tau, \tau)\rmd \tau \\
        & - \int_{t}^{s} \bfomega(\bfx_\tau, \tau)\rmd \tau.
    \end{aligned}
    \end{equation}
\end{assumption}

\begin{assumption}
    Decompose $\bfeps_{\theta}(\bfx_{m}, m)$ in \cref{eq:sup_amed} into two components as in Assumption \ref{assump:assump2} that parallel and perpendicular to the plane where $\{\tilde{\bfx}_{\tau}\}_{\tau=\epsilon}^{T}$ lies:
    \begin{equation}
        \bfeps_{\theta}(\bfx_{m}, m) = \bfeps_{\theta}^\parallel(\bfx_{m}, m) + \bfeps_{\theta}^\perp(\bfx_{m}, m).
    \end{equation}
    Assume that the parallel component is optimally learned and for the perpendicular component, we have
    \begin{equation}
        \left\| c(s-t)\bfeps_{\theta}^\perp(\bfx_{m}, m) \right\|_2 \le \left\| \int_{t}^{s} \bfomega(\bfx_\tau, \tau)\rmd \tau \right\|_2.
    \end{equation}
\end{assumption}

\begin{assumption}
    \label{assump:assump4}
    There exists such a $g(\tau)$ s.t. with high probability that
    \begin{equation}
        \label{eq:sup_bound}
        \left\| \int_{t}^{s} \bfomega(\tilde{\bfx}_\tau, \tau) \rmd \tau \right\|_2 \le \left\| \bfz_s \right\|_2.
    \end{equation}
\end{assumption}

\begin{lemma}
    \label{lemma:lemma1}
    Under assumption \ref{assump:assump1} and \ref{assump:assump4}, let $g(\tau) = f(\tau) / \sqrt{d}$, then $\bfz_s$ concentrates at a thin shell with radius 
    \begin{equation}
        r(s,t) = \frac{a}{\sqrt{b}}\sqrt{\frac{1}{1+e^{x}} \bigg|_{bs}^{bt} + \frac{b}{4}(t-s)}. 
    \end{equation}
\end{lemma}
\begin{proof}
    Since the SDE \cref{eq:sde_assump4} has zero drift coefficient, its perturbation kernel $p(\bfz_s|\bfz_t=\bf0)$ is a Gaussian with zero mean~\cite{sarkka2019applied}. The covariance $\mathbf{P}(s,t)$ is given by
    \begin{align}
        \mathbf{P}(s,t) 
        &= \int_s^t g^2(\tau) \rmd \tau \mathbf{I} \\
        &= \frac{a^2}{d} \int_s^t \left(\frac{1}{1+e^{-b\tau}} - \frac{1}{2}\right)^2 \rmd \tau \mathbf{I} \\
        &= \underbrace{\frac{a^2}{bd} \left( \frac{1}{1+e^{x}} \bigg|_{bs}^{bt} + \frac{b}{4}(t-s) \right)}_{{\sigma^2(s,t)}} \mathbf{I}.
    \end{align}
    By the well-known concentration of measure~\cite{vershynin2018high}, there exists a constant $c > 0$ s.t. for any $h \ge 0$, we have
    \begin{equation}
        \bbP\left(\left| \left\| \bfz_s \right\|_2 - \left|\sigma(s,t)\right| \sqrt{d}\right| \ge h\right) \le 2e^{-ch^2}
    \end{equation}
    which complete the proof.
\end{proof}

\begin{proposition}
    Given $\epsilon \le s < t \le T$, under the assumptions and \cref{lemma:lemma1} above, with high probability we have 
    \begin{equation}
        \left\|\bfx_{s}-\bfx_{s}^A\right\|_2 \le f(s) + f(t) + r(s, t).
    \end{equation}
\end{proposition}
\begin{proof}
    Under assumptions and \cref{lemma:lemma1} above, we have
    \begin{align}
        &\left\|\bfx_{s} - \bfx_{s}^A\right\|_2 \\
        &\le \left\|\bfx_{s} - \tilde{\bfx}_{s}\right\|_2 + \left\|\tilde{\bfx}_{s} - \bfx_{s}^A\right\|_2 \\
        &\le f(s) + \Bigl|\Bigl|\bfx_{t} - \tilde{\bfx}_{t} + c(s - t) \bfeps_{\theta}(\bfx_{m}, m) \\
        &\quad - \int_{t}^{s} \bfeps_{\theta}(\bfx_\tau, \tau) + \bfomega(\bfx_\tau, \tau)\rmd \tau\Bigl|\Bigl|_2 \\
        &\le f(s) + f(t) + \Bigl|\Bigl| c(s - t) \bfeps_{\theta}^\perp(\bfx_{m}, m) \\
        &\quad + c(s - t) \bfeps_{\theta}^\parallel(\bfx_{m}, m) - \int_{t}^{s} (\bfeps_{\theta}+\bfomega)(\bfx_\tau, \tau)\rmd \tau\Bigl|\Bigl|_2 \\
        &\le f(s) + f(t) + \left\|c(s - t)\bfeps_{\theta}^\perp(\bfx_{m}, m)\right\|_2 \\
        &\le f(s) + f(t) + \left\| \int_{t}^{s} \bfomega(\bfx_\tau, \tau)\rmd \tau \right\|_2 \\
        &\le f(s) + f(t) + r(s, t)
    \end{align}
    with high probability.
\end{proof}

\begin{figure*}
  \centering
  \begin{subfigure}[b]{0.48\linewidth}
      \includegraphics[width=\linewidth]{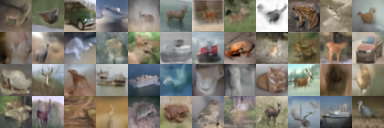}
      \caption{Baseline iPNDM solver. FID = 47.98.}
  \end{subfigure}
  \begin{subfigure}[b]{0.48\linewidth}
      \includegraphics[width=\linewidth]{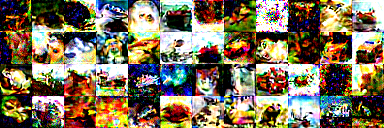}
      \caption{Baseline DPM-Solver-2. FID = 155.7.}
  \end{subfigure}
  \begin{subfigure}[b]{0.48\linewidth}
      \includegraphics[width=\linewidth]{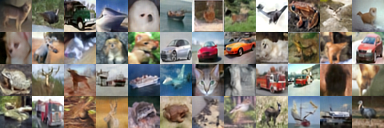}
      \caption{AMED-Plugin applied on iPNDM solver. FID = 10.81.}
  \end{subfigure}
  \begin{subfigure}[b]{0.48\linewidth}
      \includegraphics[width=\linewidth]{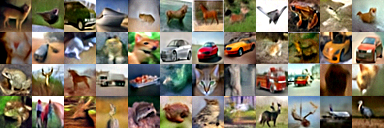}
      \caption{AMED-Solver. FID = 18.49.}
  \end{subfigure}
  \caption{Uncurated samples on CIFAR10 32$\times$32 with 3 NFE.}
  \label{fig:sup_grid_cifar10_3}
\end{figure*}

\begin{figure*}
  \centering
  \begin{subfigure}[b]{0.48\linewidth}
      \includegraphics[width=\linewidth]{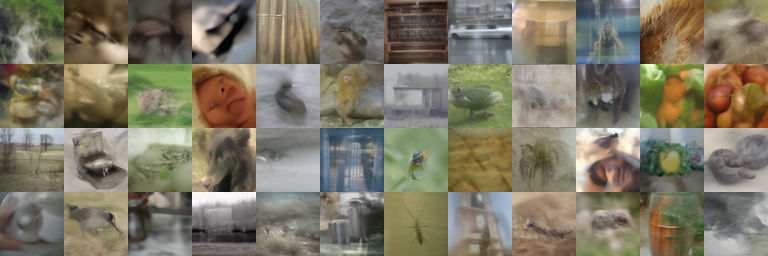}
      \caption{Baseline iPNDM solver. FID = 58.53.}
  \end{subfigure}
  \begin{subfigure}[b]{0.48\linewidth}
      \includegraphics[width=\linewidth]{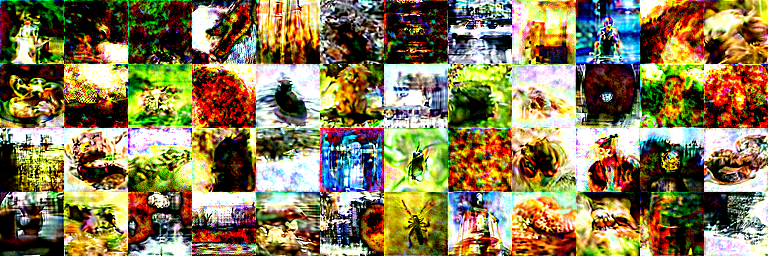}
      \caption{Baseline DPM-Solver-2. FID = 140.2.}
  \end{subfigure}
  \begin{subfigure}[b]{0.48\linewidth}
      \includegraphics[width=\linewidth]{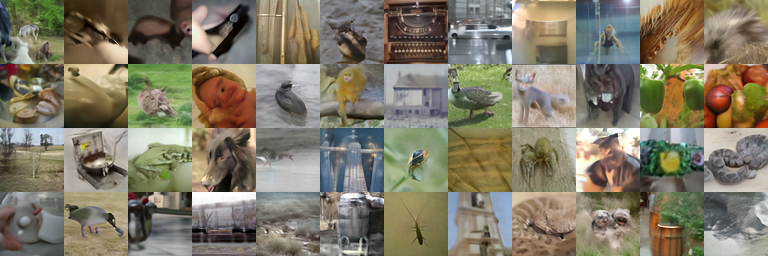}
      \caption{AMED-Plugin applied on iPNDM solver. FID = 28.06.}
  \end{subfigure}
  \begin{subfigure}[b]{0.48\linewidth}
      \includegraphics[width=\linewidth]{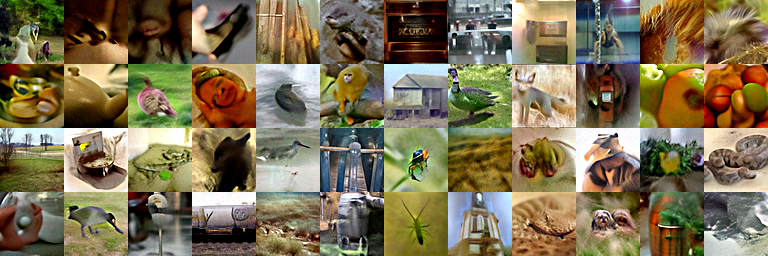}
      \caption{AMED-Solver. FID = 38.10.}
  \end{subfigure}
  \caption{Uncurated samples on Imagenet 64$\times$64 with 3 NFE.}
  \label{fig:sup_grid_imagenet64_3}
\end{figure*}

\begin{figure*}
  \centering
  \begin{subfigure}[b]{0.48\linewidth}
      \includegraphics[width=\linewidth]{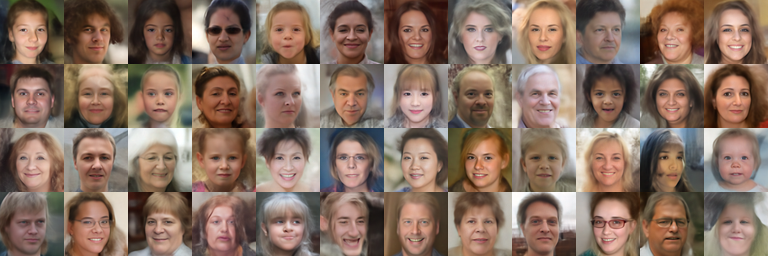}
      \caption{Baseline iPNDM solver. FID = 45.98.}
  \end{subfigure}
  \begin{subfigure}[b]{0.48\linewidth}
      \includegraphics[width=\linewidth]{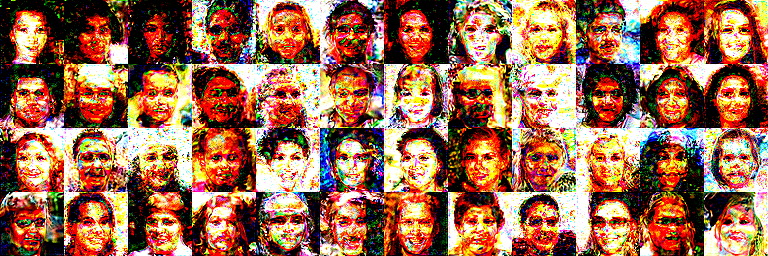}
      \caption{Baseline DPM-Solver-2. FID = 266.0.}
  \end{subfigure}
  \begin{subfigure}[b]{0.48\linewidth}
      \includegraphics[width=\linewidth]{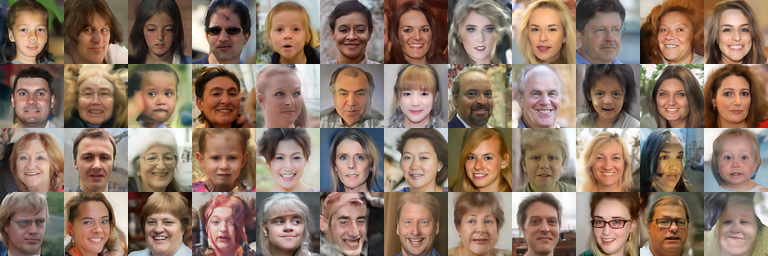}
      \caption{AMED-Plugin applied on iPNDM solver. FID = 26.87.}
  \end{subfigure}
  \begin{subfigure}[b]{0.48\linewidth}
      \includegraphics[width=\linewidth]{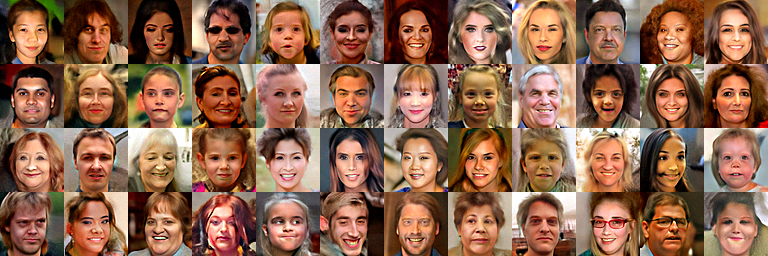}
      \caption{AMED-Solver. FID = 47.31.}
  \end{subfigure}
  \caption{Uncurated samples on FFHQ 64$\times$64 with 3 NFE.}
  \label{fig:sup_grid_ffhq64_3}
\end{figure*}
\clearpage

\begin{figure*}
  \centering
  \begin{subfigure}[b]{0.48\linewidth}
      \includegraphics[width=\linewidth]{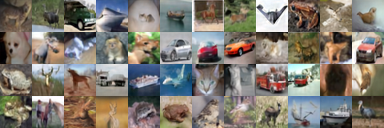}
      \caption{Baseline iPNDM solver. FID = 13.59.}
  \end{subfigure}
  \begin{subfigure}[b]{0.48\linewidth}
      \includegraphics[width=\linewidth]{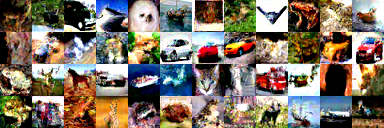}
      \caption{Baseline DPM-Solver-2. FID = 57.30.}
  \end{subfigure}
  \begin{subfigure}[b]{0.48\linewidth}
      \includegraphics[width=\linewidth]{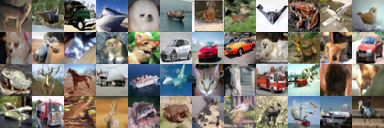}
      \caption{AMED-Plugin applied on iPNDM solver. FID = 6.61.}
  \end{subfigure}
  \begin{subfigure}[b]{0.48\linewidth}
      \includegraphics[width=\linewidth]{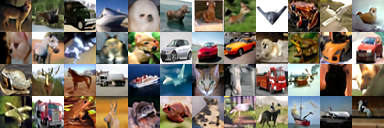}
      \caption{AMED-Solver. FID = 7.59.}
  \end{subfigure}
  \caption{Uncurated samples on CIFAR10 32$\times$32 with 5 NFE.}
  \label{fig:sup_grid_cifar10_5}
\end{figure*}

\begin{figure*}
  \centering
  \begin{subfigure}[b]{0.48\linewidth}
      \includegraphics[width=\linewidth]{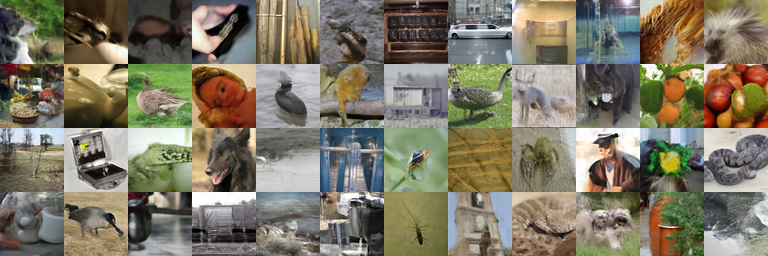}
      \caption{Baseline iPNDM solver. FID = 18.99.}
  \end{subfigure}
  \begin{subfigure}[b]{0.48\linewidth}
      \includegraphics[width=\linewidth]{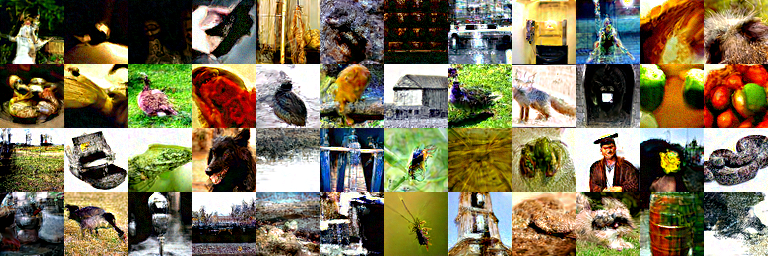}
      \caption{Baseline DPM-Solver-2. FID = 42.41.}
  \end{subfigure}
  \begin{subfigure}[b]{0.48\linewidth}
      \includegraphics[width=\linewidth]{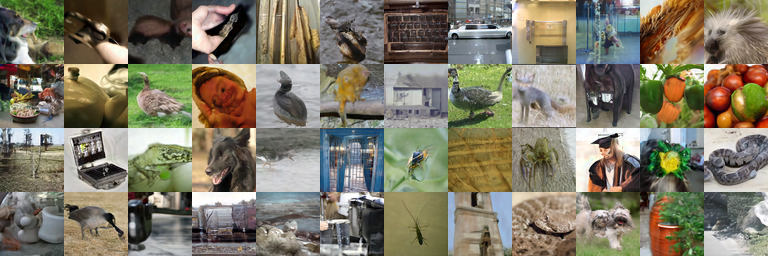}
      \caption{AMED-Plugin applied on iPNDM solver. FID = 13.83.}
  \end{subfigure}
  \begin{subfigure}[b]{0.48\linewidth}
      \includegraphics[width=\linewidth]{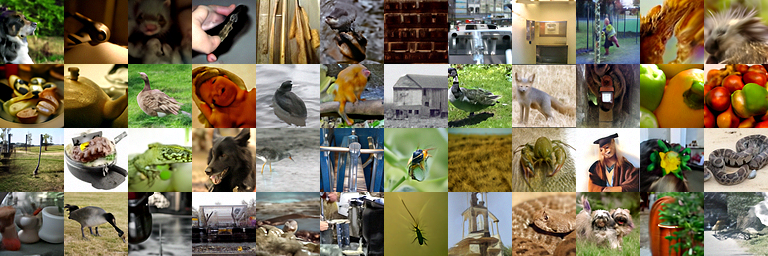}
      \caption{AMED-Solver. FID = 10.74.}
  \end{subfigure}
  \caption{Uncurated samples on Imagenet 64$\times$64 with 5 NFE.}
  \label{fig:sup_grid_imagenet64_5}
\end{figure*}

\begin{figure*}
  \centering
  \begin{subfigure}[b]{0.48\linewidth}
      \includegraphics[width=\linewidth]{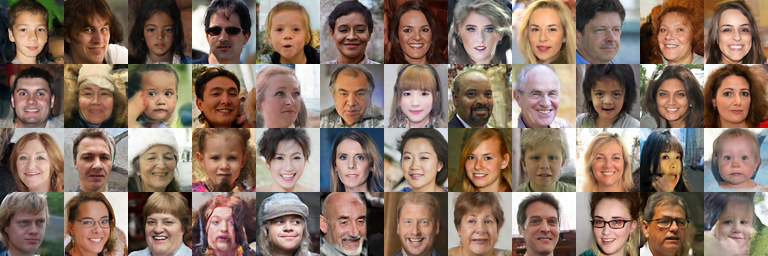}
      \caption{Baseline iPNDM solver. FID = 17.17.}
  \end{subfigure}
  \begin{subfigure}[b]{0.48\linewidth}
      \includegraphics[width=\linewidth]{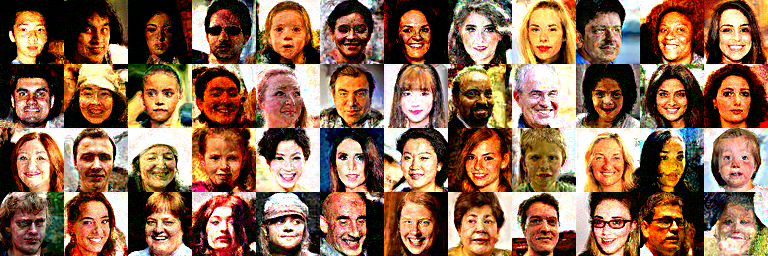}
      \caption{Baseline DPM-Solver-2. FID = 87.10.}
  \end{subfigure}
  \begin{subfigure}[b]{0.48\linewidth}
      \includegraphics[width=\linewidth]{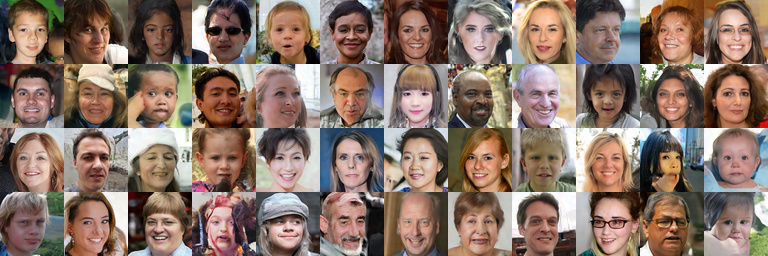}
      \caption{AMED-Plugin applied on iPNDM solver. FID = 12.49.}
  \end{subfigure}
  \begin{subfigure}[b]{0.48\linewidth}
      \includegraphics[width=\linewidth]{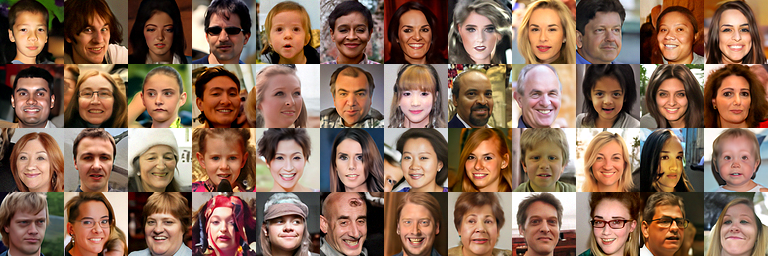}
      \caption{AMED-Solver. FID = 14.80.}
  \end{subfigure}
  \caption{Uncurated samples on FFHQ 64$\times$64 with 5 NFE.}
  \label{fig:sup_grid_ffhq64_5}
\end{figure*}
\clearpage

\end{document}